\documentclass[twoside]{article}

\usepackage[preprint]{aistats2027}

\usepackage[round]{natbib}

\usepackage{graphicx}

\usepackage{amsmath}
\usepackage{amssymb}
\usepackage{amsfonts}
\usepackage{bm}          

\usepackage{amsthm}      

\theoremstyle{plain}
\newtheorem{theorem}{Theorem}
\newtheorem{lemma}{Lemma}

\DeclareMathOperator*{\argmin}{arg\,min}

\theoremstyle{definition}
\newtheorem{assumption}{Assumption}
\newtheorem{definition}{Definition}

\theoremstyle{remark}
\newtheorem{remark}{Remark}

\usepackage{booktabs}

\usepackage{tabularx}

\usepackage{enumitem}
\usepackage{multirow}
\usepackage{url}

\usepackage[colorlinks=true, 
            citecolor=blue, 
            linkcolor=blue, 
            urlcolor=blue]{hyperref}

\begin{document}



\runningauthor{Alexey Zaytsev, Diana Koldasbayeva, Egor Serov}
\twocolumn[

\aistatstitle{Reliable Risk Estimation with Kernel Mean Matching under Dependent Sampling}

\aistatsauthor{ Alexey Zaytsev$^*$ \And Diana Koldasbayeva$^*$ \And Egor Serov$^{*,\dagger}$ }

\aistatsaddress{ Applied AI Institute } ]

{\renewcommand{\thefootnote}{}
\footnotetext{$^*$Equal contribution.}
\footnotetext{$^\dagger$e.serov@applied-ai.ru}
}

\begin{abstract}
Kernel Mean Matching (KMM) estimates risk under covariate shift by reweighting labelled source samples so that their kernel mean embedding matches that of unlabeled target samples. 
Existing finite-sample guarantees are usually stated for independent samples, although many applications involve temporal, spatial, or clustered dependence. We give a finite-sample KMM analysis under dependent sampling by replacing nominal sample sizes with effective sample sizes in the concentration steps used by the classical proof. 
The approach uses Bernstein inequalities for geometrically strongly mixing processes and derives a model-based effective sample size. 
The resulting bound separates overlap, dependence, and smoothness effects with specific constants derived under a Gaussian process assumption. 
This analysis yields an additional tradeoff: smoother losses preserve more input dependence and therefore reduce effective sample size, while also reducing function complexity and empirical-to-population risk gaps. 
Applying the result to prediction yields confidence intervals that achieve near-nominal coverage in synthetic experiments and substantially reduce undercoverage on real monitoring data.

\end{abstract}

\section{Introduction}
\label{sec:intro}

Machine learning models often face distribution shift: labelled training samples come from a source distribution, while predictions are required under a different target distribution.
This mismatch complicates reliable evaluation, particularly in scientific domains where observations are collected along spatial transects, temporal monitoring campaigns, or clustered sampling schemes \citep{roberts2017cross}.
Specifically, under \emph{covariate shift}, the conditional response law remains unchanged, while the marginal covariate distribution changes. In ecological and geospatial modelling, spatial transects and temporal monitoring campaigns produce structured observations, while uneven sampling may leave labelled source data systematically unrepresentative of the deployment region \citep{koldasbayeva2025foundation}.

A standard approach to estimating target risk from labelled source samples and unlabelled target covariates is to reweight source losses by the ratio of target to source covariate densities \citep{shimodaira2000improving,sugiyama2007covariate}. Estimating this ratio can be unstable when overlap is poor, source observations are sparse in high-ratio regions, or the ratio model is misspecified \citep{kouw2019review}. Kernel Mean Matching (KMM) provides an alternative to explicit density-ratio estimation: it computes source weights that align the empirical source and target kernel mean embeddings in a reproducing kernel Hilbert space (RKHS) \citep{huang2006correcting,gretton2009covariate,serov2026kernel}.

Despite KMM's practical success, its standard finite-sample theory relies on independent observations~\citep{yu2012analysis}. This assumption is routinely violated in real-world settings characterized by temporal persistence, spatial autocorrelation, environmental gradients, or cluster sampling. For these dependent regimes, treating the data as independent leads to overconfidence, as the nominal sample size substantially overstates the true amount of statistical information.

This paper develops a dependent-sampling foundation for KMM risk-estimation guarantees. The guiding principle of our approach is that the standard i.i.d. proof structure remains valid if standard concentration inequalities are replaced by their dependent counterparts, substituting the nominal sample size $n$ with an \emph{effective sample size} $n_{\mathrm{eff}}$. Although $n_{\mathrm{eff}}$ is a long-standing cornerstone for diagnosing correlation in Markov Chain Monte Carlo \citep{kass1998markov, vehtari2019rank}, bounding weight degradation in importance sampling \citep{martino2017effective}, and estimating spatial degrees of freedom \citep{bretherton1999effective}, our work formally integrates this concept into non-parametric RKHS generalization bounds \citep{gretton2012kernel} for domain adaptation. This allows us to derive risk-estimation bounds under covariate shift and non-i.i.d. sampling that decompose the distinct roles of distribution overlap, data dependence, and function smoothness. Furthermore, our theoretical analysis reveals an important underlying tradeoff: smoother target or loss functions have lower statistical complexity, but when evaluated along a correlated covariate sequence, they preserve more dependence, which inherently reduces the effective sample size.

Specifically, our main contributions are as follows:
\begin{enumerate}
    \item We formulate an effective-concentration condition, based on a Bernstein inequality for strongly $\beta$-mixing sequences, under which the classical finite-sample KMM guarantee naturally extends to dependent data.
    
    \item We explicitly decompose the generalization bottleneck into three components: independent noise, density-ratio overlap, and dependent sampling-sequence persistence.

    \item We characterize how covariate dependence and response smoothness determine effective sample size, deriving a closed-form GP-induced $n_{\mathrm{eff}}$ that yields a practical dependence correction for KMM risk estimation.


    \item We validate the GP-induced effective-size mechanism in controlled synthetic experiments and evaluate the resulting confidence correction on three real-world monitoring datasets, where it consistently reduces empirical undercoverage relative to the i.i.d. interval.
    
\end{enumerate}

\section{Related Work}
\label{sec:related}

Kernel Mean Matching (KMM) was originally introduced as a non-parametric approach to correct for sample selection bias and covariate shift \citep{huang2006correcting, gretton2009covariate}. More broadly, learning and evaluation under covariate shift are typically addressed via importance-weighted likelihoods, cross-validation schemes, and adapted risk estimators \citep{shimodaira2000improving, sugiyama2007covariate, kouw2019review}. The foundational finite-sample analysis by \citet{yu2012analysis} demonstrates that KMM effectively bounds the target expectation error, provided the regression function lies within the RKHS and the source and target samples are independent. Our theoretical analysis adopts a similar proof structure but generalises it by replacing standard independent concentration steps with effective concentration inequalities for dependent data.


Dependence is ubiquitous in environmental and geospatial data, where nearby observations in space or time often carry substantially less independent information than their nominal sample size suggests. Ignoring this structure can lead to overly optimistic evaluation and
confidence estimates, motivating dependence-aware validation and transfer procedures \citep{roberts2017cross, koldasbayeva2025foundation}. Recent work has also investigated KMM directly under spatial distribution shift, showing empirically that dependence and source--target mismatch can jointly affect transfer-risk estimation \citep{serov2026kernel}. Our contribution complements this empirical line of work by providing a finite-sample analysis of how dependent sampling modifies the KMM risk-estimation guarantee and by making the resulting loss of information explicit through an effective sample size.

The dependent-data component of our framework builds upon concentration inequalities for strongly mixing processes \citep{rio1993covariance, rio1999theorie, bosq2012nonparametric}, specifically leveraging a Bernstein bound for geometrically $\beta$-mixing sequences \citep{merlevede2009bernstein}. We translate these bounds into an \emph{effective sample size} ($n_{\mathrm{eff}}$) perspective, creating a bridge to other fields where data dependence is a central challenge. For instance, in Markov Chain Monte Carlo (MCMC), $n_{\mathrm{eff}}$ is the canonical diagnostic for chain convergence \citep{kass1998markov, vehtari2019rank}. While the MCMC literature has explored multivariate ESS for finite-dimensional outputs \citep{vats2019multivariate}, our work extends this concept by defining effective sample sizes for infinite-dimensional RKHS embeddings. Furthermore, KMM shares conceptual roots with Importance Sampling (IS), where weight degradation due to distribution discrepancy is actively quantified via ESS \citep{martino2017effective}; in our KMM bound, the domain overlap penalty $B$ formally captures this identical discrepancy bottleneck.


Finally, Gaussian-process models provide a convenient probabilistic framework for relating dependence in the inputs to dependence in a smooth response function \citep{rasmussen2003gaussian, srinivas2009gaussian}. We use this connection to obtain an explicit, model-based effective sample size for KMM risk estimation. Rather than treating $n_{\mathrm{eff}}$ as an abstract concentration parameter, the GP construction expresses it in terms of the correlation structure of the covariates and the smoothness scale of the response. This provides a practical link between the dependent-data concentration analysis and the confidence correction
evaluated in our experiments.

Table~\ref{tab:positioning} summarizes our positioning relative to this diverse theoretical landscape.

\begin{table}[t]
\centering
\caption{Positioning relative to related theory.}
\label{tab:positioning}
\vspace{1mm}
\footnotesize
\begin{tabular}{@{}p{0.28\linewidth}>{\raggedright}p{0.32\linewidth}>{\raggedright\arraybackslash}p{0.34\linewidth}@{}}
\toprule
\textbf{Reference} & \textbf{Setting} & \textbf{Relation to Ours} \\
\midrule
\citep{huang2006correcting, gretton2009covariate} & Covariate shift KMM & Introduces KMM objective \\
\addlinespace[2pt]
\citep{yu2012analysis} & Finite-sample bound & i.i.d. proof template \\
\addlinespace[2pt]
\citep{roberts2017cross, koldasbayeva2025foundation} & Spatio-temporal CV & Dependence motivation \\
\addlinespace[2pt]
\citep{merlevede2009bernstein} & $\beta$-mixing processes & Scalar concentration tools \\
\addlinespace[3pt]
\textbf{This paper} & \textbf{Dependent KMM} & \textbf{Explicit bounds via $n_{\mathrm{eff}}$} \\
\bottomrule
\end{tabular}
\end{table}

\section{Theory}
\label{sec:setup}

\subsection{Problem Setup}
\label{sec:problem_setup}

Consider a labelled source sequence $\{(\mathbf{X}_i^{\mathrm{tr}}, Y_i^{\mathrm{tr}})\}_{i=1}^{n_{\mathrm{tr}}}$ with joint distribution $P_{\mathrm{tr}}$ over $\mathcal{X} \times \mathcal{Y}$, having covariate marginal $P_{\mathrm{tr}}^{\mathcal{X}}$, and an unlabeled target sequence $\{\mathbf{X}_j^{\mathrm{te}}\}_{j=1}^{n_{\mathrm{te}}}$ with marginal $P_{\mathrm{te}}$ over $\mathcal{X} \subset \mathbb{R}^d$. The covariate space $\mathcal{X}$ is compact and measurable, and responses are normalized to $Y \in [0, 1]$ 

\begin{assumption}[Covariate Shift]
\label{ass:covariate_shift}
The conditional response law is invariant across domains: $P_{\mathrm{tr}}(dy \mid \mathbf{x}) = P_{\mathrm{te}}(dy \mid \mathbf{x})$.
\end{assumption}
Under Assumption~\ref{ass:covariate_shift}, we define the invariant regression function $m(\mathbf{x}) := \mathbb{E}[Y \mid \mathbf{X}=\mathbf{x}]$. Our objective is to estimate the target expectation $\mu_{\mathrm{te}} := \mathbb{E}_{P_{\mathrm{te}}}[Y] = \int_{\mathcal{X}} m(\mathbf{x})\, P_{\mathrm{te}}(d\mathbf{x})$.

\begin{assumption}[Conditionally Independent Residuals]
\label{ass:independent_residuals}
Conditional on all observed covariates $(\mathbf{X}^{\mathrm{tr}}, \mathbf{X}^{\mathrm{te}})$, the residual errors $\epsilon_i = Y_i^{\mathrm{tr}} - m(\mathbf{X}_i^{\mathrm{tr}})$ are mutually independent, mean-zero random variables.
\end{assumption}

\begin{remark}[Scope of Residual Independence]
\label{rem:residuals}
Assumption~\ref{ass:independent_residuals} isolates the structural challenge of covariate shift from idiosyncratic label noise. It allows us to prove Theorem~\ref{thm:kmm-dependent} at the nominal parametric rate. While this is a standard theoretical abstraction in non-parametric regression under dependent designs \citep{bosq2012nonparametric, yu2012analysis}, real-world spatial datasets may contain unobserved confounders causing the residuals to exhibit their own autocorrelation. 
\end{remark}

\begin{assumption}[Bounded Density Ratio and Compact Support]
\label{ass:bounded_ratio}
The covariate space $\mathcal{X} \subset \mathbb{R}^d$ is compact and measurable. The derivative $\omega^*(\mathbf{x}) = (dP_{\mathrm{te}} / dP_{\mathrm{tr}}^{\mathcal{X}})(\mathbf{x})$ exists and satisfies $0 \le \omega^*(\mathbf{x}) \le B < \infty$ for all $\mathbf{x} \in \mathcal{X}$.
\end{assumption}

\begin{assumption}[Moderate-Deviation Sample Size]
\label{ass:burn_in}
For a given failure probability $\delta \in (0, 1)$, the source sample size $n_{\mathrm{tr}}$ is sufficiently large relative to the sequence mixing rate to satisfy the moderate-deviation condition:
\begin{equation}
\label{eq:sample_burnin}
\frac{n_{\mathrm{tr}}}{(\log n_{\mathrm{tr}})^4} \ge \frac{B^2 K_2^2}{2 K_1 \sigma_{v,\infty}^2} \log\frac{6}{\delta},
\end{equation}
where $K_1, K_2 > 0$ and $\sigma_{v,\infty}^2$ are defined in Appendix~\ref{app:lemma1}.
\end{assumption}

\paragraph{Kernel Mean Matching (KMM).}
Let $\mathcal{H}$ be an RKHS equipped with a feature map $\Phi: \mathcal{X} \to \mathcal{H}$, and assume the features are uniformly bounded: $\sup_{\mathbf{x} \in \mathcal{X}} \|\Phi(\mathbf{x})\|_{\mathcal{H}} \le C$. Rather than explicitly estimating $\omega^*(\mathbf{x})$, KMM \citep{huang2006correcting, yu2012analysis} computes weights $\hat{\mathbf{w}} = (\hat{w}_1, \dots, \hat{w}_{n_{\mathrm{tr}}})^\top$ aligning the empirical RKHS means $\hat{\mu}_{\Phi}^{\mathrm{tr}}(\mathbf{w}) := \frac{1}{n_{\mathrm{tr}}} \sum_{i=1}^{n_{\mathrm{tr}}} w_i \Phi(\mathbf{X}_i^{\mathrm{tr}})$ and $\hat{\mu}_{\Phi}^{\mathrm{te}} := \frac{1}{n_{\mathrm{te}}} \sum_{j=1}^{n_{\mathrm{te}}} \Phi(\mathbf{X}_j^{\mathrm{te}})$:

\begin{equation}
\label{eq:kmm}
\hat{\mathbf{w}} = \argmin_{\mathbf{w} \in [0, B]^{n_{\mathrm{tr}}}} \left\| \hat{\mu}_{\Phi}^{\mathrm{tr}}(\mathbf{w}) - \hat{\mu}_{\Phi}^{\mathrm{te}} \right\|_{\mathcal{H}}.
\end{equation}
The reweighted target mean estimator is then given by:
\begin{equation}
\label{eq:kmm-estimator}
\hat{\mu}_{\mathrm{KMM}} = \tfrac{1}{n_{\mathrm{tr}}} \sum_{i=1}^{n_{\mathrm{tr}}} \hat{w}_i Y_i^{\mathrm{tr}}.
\end{equation}
For $Y$ being a loss function we obtain a risk estimate.

\subsection{Main Dependent KMM Bound}
\label{sec:main_bound}

The classical KMM guarantee \citep{yu2012analysis} relies on Hoeffding bounds for independent samples. To adapt this to dependent sampling, we distill the Bernstein inequality for geometrically $\beta$-mixing processes \citep{merlevede2009bernstein} into a deviation-independent effective sample size $n_{\mathrm{eff}}^{\mathrm{tr}}$ (detailed in Appendix~\ref{app:lemma1}). To control the empirical RKHS feature-matching error, we apply Berbee's coupling \citep{berbee1979random} and Pinelis' inequality \citep{pinelis1994optimum} to derive a Hilbert-space concentration bound governed by a blocked effective sample size $n_{\mathrm{eff, RKHS}}$ (derived in Appendix~\ref{app:lemma2}). As a result, we decompose the KMM target estimation error into three distinct operational components.

\begin{theorem}[Non-Asymptotic Dependent KMM Bound]
\label{thm:kmm-dependent}
Suppose Assumptions \ref{ass:covariate_shift}--\ref{ass:bounded_ratio} hold, and Assumption~\ref{ass:burn_in} is satisfied for a given failure probability $\delta \in (0,1)$. Assume the source and target covariate sequences are strictly stationary and geometrically $\beta$-mixing. Let the invariant regression function satisfy $m \in \mathcal{H}$, and let $v(\mathbf{x}) = \omega^*(\mathbf{x})m(\mathbf{x})$ have a long-run variance proxy $\sigma_{v,\infty}^2$.

Then, with probability at least $1-\delta$:
\begin{align}
\label{eq:main-bound}
&\left| \hat{\mu}_{\mathrm{KMM}} - \mu_{\mathrm{te}} \right| \le 
\underbrace{ B \sqrt{\frac{\log(6/\delta)}{2n_{\mathrm{tr}}}} }_{\mathcal{E}_{\mathrm{noise}}} 
+ \underbrace{ B \sqrt{\frac{\log(6/\delta)}{2n_{\mathrm{eff}}^{\mathrm{tr}}}} }_{\mathcal{E}_{\mathrm{shift}}} \nonumber \\
&+ \underbrace{ \mathcal{O}\left( \|m\|_{\mathcal{H}} \left( \frac{B}{\sqrt{n_{\mathrm{eff, RKHS}}^{\mathrm{tr}}}} + \frac{1}{\sqrt{n_{\mathrm{eff, RKHS}}^{\mathrm{te}}}} \right) \right) }_{\mathcal{E}_{\mathrm{MMD}}},
\end{align}
where $n_{\mathrm{eff}}^{\mathrm{tr}}$ is the scalar effective sample size evaluated for $v(\mathbf{X})$ derived in Appendix~\ref{app:lemma1}. The quantities $n_{\mathrm{eff, RKHS}}^{\mathrm{tr}}$ and $n_{\mathrm{eff, RKHS}}^{\mathrm{te}}$ scale as $\mathcal{O}(n / \log n)$ and denote the \emph{RKHS effective sample sizes} arising from the geometric blocking penalty. The exact non-asymptotic constants for $\mathcal{E}_{\mathrm{MMD}}$ are detailed in Appendix~\ref{app:theorem1}. If the sample size is too small to satisfy Assumption~\ref{ass:burn_in}, the bound remains valid by substituting the deviation-dependent effective sample size $n_{\mathrm{eff}}^{\mathrm{tr}}(t)$ shown in Appendix~\ref{app:lemma1}.
\end{theorem}

\paragraph{Interpretation and The Bottleneck of Dependence.}
Theorem~\ref{thm:kmm-dependent} isolates the three fundamental challenges of risk estimation under covariate shift. This decomposition reveals how these error terms interact under strong data dependence ($n_{\mathrm{eff}}^{\mathrm{tr}} \ll n_{\mathrm{tr}}$):
\begin{enumerate}
    \item \textbf{Independent Noise ($\mathcal{E}_{\mathrm{noise}}$):} From Assumption~\ref{ass:independent_residuals}, the conditionally independent idiosyncratic label noise concentrates at nominal rate $\mathcal{O}(1/\sqrt{n_{\mathrm{tr}}})$. Under strong spatial or temporal dependence, this noise averages out rapidly and becomes statistically negligible compared to the other terms.
    \item \textbf{Dependent Sampling Bottleneck ($\mathcal{E}_{\mathrm{shift}}$):} The error of evaluating the true density ratio over correlated covariates is penalized. Because the sampling process repeatedly observes similar regions of the covariate space, this term scales with the deflated effective sample size $\mathcal{O}(1/\sqrt{n_{\mathrm{eff}}^{\mathrm{tr}}})$. When $n_{\mathrm{eff}}^{\mathrm{tr}} \ll n_{\mathrm{tr}}$, $\mathcal{E}_{\mathrm{shift}}$ completely dominates $\mathcal{E}_{\mathrm{noise}}$, making sequence persistence the primary bottleneck of scalar generalization.
    \item \textbf{RKHS Matching ($\mathcal{E}_{\mathrm{MMD}}$) and Overlap Penalty ($B$):} The feature-matching error decays at the blocked rate governed by the exact RKHS effective sample sizes. While this worst-case blocking penalty appears mathematically less severe than the scalar long-run variance deflation, both scalar terms ($\mathcal{E}_{\mathrm{noise}}$ and $\mathcal{E}_{\mathrm{shift}}$) are multiplicatively scaled by $B$. Poor source-target overlap (manifested as large $B$) quadratically inflates the required effective sample size.
\end{enumerate}

\paragraph{Application to Risk Estimation}
While Theorem~\ref{thm:kmm-dependent} is formulated for a general bounded scalar response $Y \in [0, 1]$, its primary operational application is target risk evaluation. Specifically, let $g: \mathcal{X} \to \mathcal{Y}_{\mathrm{pred}}$ be a fixed hypothesis, and let $\ell: \mathcal{Y}_{\mathrm{pred}} \times \mathcal{Y} \to [0, 1]$ be a bounded loss function. By setting the response to the sample loss $Y_i^{\mathrm{tr}} = \ell(g(\mathbf{X}_i^{\mathrm{tr}}), Y_i^{\mathrm{tr},\mathrm{label}})$, the true target risk is given by $R_{\mathrm{te}}(g) := \mathbb{E}_{P_{\mathrm{te}}}[\ell(g(\mathbf{X}), Y_{\mathrm{label}})]$, and its reweighted KMM estimate is $\hat{R}_{\mathrm{KMM}}(g) := \frac{1}{n_{\mathrm{tr}}} \sum_{i=1}^{n_{\mathrm{tr}}} \hat{w}_i Y_i^{\mathrm{tr}}$. Under Assumption~\ref{ass:covariate_shift}, the invariant regression function represents the conditional risk surface $m_\ell(\mathbf{x}) := \mathbb{E}[\ell(g(\mathbf{X}), Y_{\mathrm{label}}) \mid \mathbf{X}=\mathbf{x}]$. Assuming $m_\ell \in \mathcal{H}$, Theorem~\ref{thm:kmm-dependent} directly bounds the risk generalization error $\left| \hat{R}_{\mathrm{KMM}}(g) - R_{\mathrm{te}}(g) \right|$ by substituting $\|m\|_{\mathcal{H}} \gets \|m_\ell\|_{\mathcal{H}}$. For uniform convergence guarantees over entire hypothesis classes via Block Rademacher complexity, we refer to Appendix~\ref{app:rademacher}.

\paragraph{From Theory to Practice: A Generative Surrogate.}
Theorem~\ref{thm:kmm-dependent} establishes that under spatial or temporal dependence, KMM generalisation is governed by sequence persistence through $n_{\mathrm{eff}}^{\mathrm{tr}}$. However, as a worst-case non-asymptotic bound, it presents practical limitations: it depends on unobservable mixing constants ($K_1, K_2$), and its reliance on the RKHS norm $\|m\|_{\mathcal{H}}$ obscures how the smoothness of the regression function governs effective sample size deflation. 

To make these abstract mixing rates practically computable, we introduce a Gaussian Process (GP) covariance surrogate that yields a computable, closed-form effective sample size $n_{\mathrm{eff}}^{\mathrm{GP}}(h)$.

\subsection{Separating Overlap, Dependence, and Smoothness via GP Modeling}
\label{sec:density-smoothness}

Following core insights of Theorem~\ref{thm:kmm-dependent}, we want to analyze how the smoothness of the conditional loss surface mathematically dictates the effective sample size. By modeling the loss as a GP, we derive an closed-form effective sample size $n_{\mathrm{eff}}^{\mathrm{GP}}(h)$ that calibrates practical confidence intervals.

\paragraph{Gaussian Surrogate and Sequence Autocorrelation.}
\label{sec:gp_realization}
In Theorem~\ref{thm:kmm-dependent}, sequence persistence depends on the composite function $v(\mathbf{X}) = \omega^*(\mathbf{X})m_\ell(\mathbf{X})$. In practice, we decouple the global density ratio $\omega^*$ from the sequence autocorrelation of the loss surface. Specifically, the scale of $\omega^*$ is directly captured by the empirical squared weights $\hat{w}_i^2$ in the numerator of the variance estimator. Meanwhile, the autocorrelation across lags is driven by the smoothness of the conditional loss $m_\ell(\mathbf{x}) = \mathbb{E}[\ell(g(\mathbf{X}), Y_{\mathrm{label}}) \mid \mathbf{X}=\mathbf{x}]$

While $\|m_\ell\|_{\mathcal{H}}$ bounds the deterministic worst-case complexity in Theorem~\ref{thm:kmm-dependent}, we model the latent variation of this conditional risk surface via a zero-mean stationary GP surrogate $u \sim \mathcal{GP}(0, k_h)$ to extract a computable, data-driven length-scale $h$: $m_\ell(\mathbf{x}) = u(\mathbf{x})$. To evaluate sequence covariances analytically, we model local feature variations as stationary Gaussian $\mathbf{X}_t \sim \mathcal{N}(\bm{\mu}, \sigma_x^2 \mathbf{I}_d)$ with formal truncation to the compact domain $\mathcal{X}$ to preserve Assumption~\ref{ass:bounded_ratio}. As detailed in Appendix~\ref{app:gp_derivations}) we evaluated it with spatial or temporal correlation $r_k = \operatorname{Corr}((\mathbf{X}_s)_l, (\mathbf{X}_{s+k})_l)$.

For the squared-exponential (SE) kernel $k_h(\mathbf{x}, \mathbf{x}') = \sigma_u^2 \exp(-\|\mathbf{x}-\mathbf{x}'\|^2 / 2h^2)$, we define the ratio of input spread to function smoothness as $a_h = 2\sigma_x^2 / h^2$. Marginalizing over the covariate distribution yields the exact prior-averaged sequence autocorrelation in closed form:
\begin{equation}
\label{eq:gp-centered-rho-vector-main}
\bar{\rho}_u(k) = \frac{ \left(1 + a_h\{1 - r_k\}\right)^{-d/2} - \left(1 + a_h\right)^{-d/2} }{ 1 - \left(1 + a_h\right)^{-d/2} }.
\end{equation}
Full derivations, along with the corresponding expressions for rough Mat\'ern kernels, are provided in Appendix~\ref{app:gp_derivations}.

The resulting GP-induced finite-sample effective sample size is thus:
\begin{equation}
\label{eq:gp-improved-neff-finite}
n_{\mathrm{eff}}^{\mathrm{GP}}(h)= \frac{n_{\mathrm{tr}}}{ 1 + 2\sum_{k=1}^{n_{\mathrm{tr}}-1} \left(1 - \tfrac{k}{n_{\mathrm{tr}}}\right) \bar{\rho}_u(k) }.
\end{equation}

\paragraph{The Smoothness-Dependence Tradeoff.}
Equation~\eqref{eq:gp-improved-neff-finite} exposes a bandwidth tradeoff that is absent in standard i.i.d. estimations. For smooth functions ($h \gg \sigma_x$, so that $a_h \ll 1$), Taylor expansion of \eqref{eq:gp-centered-rho-vector-main} shows that the sequence autocorrelation  inherits the persistence of the raw covariates: $\bar{\rho}_u(k) = r_k + \mathcal{O}(a_h)$, which severely deflates $n_{\mathrm{eff}}^{\mathrm{GP}}$. Conversely, rough functions ($a_h \gg 1$) decorrelate rapidly across lags.

Importantly, this does not imply that rough functions are statistically preferable. It highlights a tension within Theorem~\ref{thm:kmm-dependent}: increasing smoothness reduces the hypothesis complexity $\|m_\ell\|_{\mathcal{H}}$, tightening the RKHS feature-matching error $\mathcal{E}_{\mathrm{MMD}}$. But inflates sequence persistence, penalizing the effective sample size and amplifying the dependent sampling error ($\mathcal{E}_{\mathrm{shift}}$).

\paragraph{Practical Confidence Calibration: The GP-Calibrated Variance Radius.}
\label{sec:refined_risk_bound}
By substituting the abstract theoretical scale $n_{\mathrm{eff}}^{\mathrm{tr}}$ in Theorem~\ref{thm:kmm-dependent} with our computable GP-realization $n_{\mathrm{eff}}^{\mathrm{GP}}(h_\ell)$, we obtain a theoretically grounded, practically evaluable measure of sampling variance. The loss length-scale $h_\ell$ is directly estimated from the source data via empirical variogram fitting without requiring access to target labels. 

Guided by this substitution, we construct an asymptotic Wald-type confidence interval. Because the target sample size $n_{\mathrm{te}}$ governs the feature-matching discrepancy $\mathcal{E}_{\mathrm{MMD}}$—which acts as a lower-order bias under stable optimization—the dominant sampling variability is isolated entirely within the reweighted source labels. Consequently, the practical confidence radius evaluates the empirical loss variance scaled by the source persistence:
\begin{equation}
\label{eq:gp-plugin-radius}
\hat{r}_{\mathrm{eff}} = z_{1 - \alpha/2} \sqrt{\frac{\widehat{\operatorname{Var}}\left[\hat{w}_i(Y_i^{\mathrm{tr}} - \hat{R}_{\mathrm{KMM}}(g))\right]}{n_{\mathrm{eff}}^{\mathrm{GP}}(h_\ell)}},
\end{equation}
where $z_{1 - \alpha/2} = \Phi^{-1}(1 - \alpha/2)$ denotes the standard normal quantile ($z_{0.975} \approx 1.96$ for the nominal $95\%$ level).

This interval treats the estimated weights $\hat{\mathbf{w}}$ as fixed post-optimization, a standard asymptotic approximation in importance weighting. 
The GP-based effective sample size adjusts the variance estimate for dependence propagated through the conditional loss surface. In settings with additional residual dependence beyond that represented by the covariates, this correction may partially absorb the resulting variance inflation, but it does not remove the conditional-independence requirement of Assumption~\ref{ass:independent_residuals} from the finite-sample guarantee. We evaluate the resulting practical interval empirically in
Section~\ref{sec:experiments}.

\section{Experiments}
\label{sec:experiments}

We design our experiments to isolate and evaluate the three fundamental effects highlighted by our theory: covariate shift through distribution overlap, data dependence through effective sample size, and function smoothness through GP-induced loss autocorrelation. We validate our approach on controlled synthetic data and benchmark it against other corrections.

\subsection{Artificial Data}
\label{sec:artificial_data}

\paragraph{Experimental Setup and Objectives}
We evaluate our dependent KMM framework on synthetic environments where ground-truth target expectations are analytically known. Covariates $\mathbf{X}_t$ are generated as continuous dependent sequences with correlation length $\ell_x$ under covariate shift ($\bm{\mu}_{\mathrm{tr}} \neq \bm{\mu}_{\mathrm{te}}$). We evaluate both a matched Gaussian Process response surface $f_0 \sim \mathcal{GP}(0, k_h)$ and a deterministic misspecified harmonic surface. We benchmark our GP-calibrated confidence radius against unweighted estimation, classical i.i.d. KMM (\textit{Nominal $n$}), raw-covariate correction (\textit{Input $n_{\mathrm{eff}}$}), Block-Bootstrap, and theoretical Oracle baselines. Complete generative specifications, solver configurations, and extensive ablation grids are detailed in Appendices~\ref{app:exp_setup}--\ref{app:param_ablations}.

\paragraph{Evaluation Metrics}
Since KMM-based variance estimators share identical weights, yielding identical point estimates, detailed in Appendix~\ref{app:full_metrics}, we focus exclusively on confidence interval quantification. We evaluate nominal $0.95$ \textit{Coverage}, considered valid if $\ge 0.925$ over $N=300$ trials, interval \textit{Radius} (half-width), and the Winkler Score (WS; \citep{winkler1972decision}), an additional scoring rule penalizing both excessive width and undercoverage.

\paragraph{Validating the Closed-Form Effective Sample Size}
\label{sec:exp_validation}

We first verify whether the closed-form $n_{\mathrm{eff}}^{\mathrm{GP}}(h)$ from \eqref{eq:gp-improved-neff-finite} accurately predicts empirical sample-mean variance. As shown in Figure~\ref{fig:neff_validation}, the variance-based empirical $n_{\mathrm{eff}}$ tracks our analytical prediction along the identity line across matched GP realizations. Because the target loss typically has a different characteristic length-scale than the raw response $Y$, the formula automatically yields estimand-specific effective sample sizes without refitting KMM weights. Validation on misspecified deterministic surfaces is provided in Appendix~\ref{app:target_det_validation}.

\begin{figure}[t]
\centering
\includegraphics[width=\columnwidth]{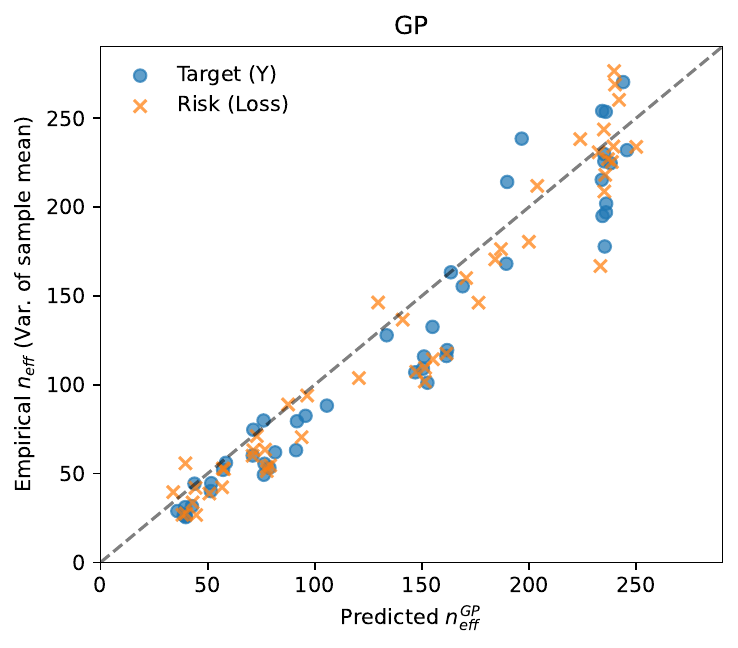}
\caption{\textbf{Validation of the GP-induced effective sample size.} The variance-based empirical $n_{\mathrm{eff}}$ aligns with our closed-form prediction (diagonal line) for both Target $Y$ (point estimation) and Risk (loss evaluation) across matched GP realizations.}
\label{fig:neff_validation}
\end{figure}

\begin{table}[!ht]
\centering
\caption{
\textbf{Confidence Interval Calibration under Dependent Sampling.}
Coverage and interval radius for target-mean and target-risk estimation.
Asterisks ($^*$) indicate significant empirical undercoverage below the
nominal 0.95 level ($<0.925$). Bold denotes the best valid practical estimator.
}
\label{tab:main-benchmark}
\vspace{1.5mm}
\resizebox{\columnwidth}{!}{%
\begin{tabular}{@{}lcccc@{}}
\toprule
& \multicolumn{2}{c}{\textbf{Target Mean ($Y$)}} & \multicolumn{2}{c}{\textbf{Target Risk ($L$)}} \\
\cmidrule(lr){2-3} \cmidrule(lr){4-5}
\textbf{Method} & \textbf{Coverage} & \textbf{Radius} & \textbf{Coverage} & \textbf{Radius} \\
\midrule
Unweighted & 0.017$^*$ & 0.0362 & 0.080$^*$ & 0.0019 \\
Nominal $n$ & 0.917$^*$ & 0.0626 & 0.853$^*$ & 0.0055 \\
Block-Bootstrap & 1.000\phantom{$^*$} & 0.4538 & 0.967\phantom{$^*$} & 0.0133 \\
Input $n_{\mathrm{eff}}$ & 1.000\phantom{$^*$} & 0.1696 & 0.987\phantom{$^*$} & 0.0150 \\
\textbf{GP $\bm{n_{\mathrm{eff}}}$ [Ours]} & \textbf{1.000}\phantom{$^*$} & \textbf{0.1380} & \textbf{0.970}\phantom{$^*$} & \textbf{0.0120} \\
\addlinespace[3pt]
\textit{Strict Bound (Theorem~\ref{thm:kmm-dependent})} & \textit{1.000}\phantom{$^*$} & \textit{5.2189} & \textit{1.000}\phantom{$^*$} & \textit{5.1220} \\
\textit{Oracle IW} & \textit{0.477}$^*$ & \textit{0.0275} & \textit{0.627}$^*$ & \textit{0.0033} \\
\textit{Oracle i.i.d.} & \textit{0.903}$^*$ & \textit{0.0257} & \textit{0.917}$^*$ & \textit{0.0024} \\
\bottomrule
\end{tabular}%
}
\end{table}

\paragraph{The Smoothness-Dependence Trade-off}
\label{sec:exp_benchmark}

Table~\ref{tab:main-benchmark} benchmarks confidence interval calibration. Applying KMM with nominal i.i.d.\ variance leads to severe overconfidence, failing on both Target $Y$ ($0.917^*$) and Target Risk ($0.853^*$). While Block-Bootstrap achieves valid coverage, it yields an excessively wide radius on Target $Y$ ($0.4538$ vs. $0.1380$) and requires expensive resampling.

These results empirically confirm the theoretical trade-off established in Section~\ref{sec:density-smoothness}. While naive \textit{Input $n_{\mathrm{eff}}$} safely restores coverage, it conservatively assumes worst-case sequence persistence ($\bar{\rho}_u(k) \approx r_k$), yielding unnecessarily wide intervals. In contrast, any response surface with finite length-scale decorrelates faster than raw input coordinates ($\bar{\rho}_u(k) \le r_k$). By analytically incorporating functional smoothness, our GP~$n_{\mathrm{eff}}$ shrinks the confidence radius by $18.6\%$ for Target $Y$ and $20.0\%$ for Target Risk relative to \textit{Input $n_{\mathrm{eff}}$}, providing the sharpest confidence bounds that maintain valid coverage and minimal WS (Table~\ref{tab:main-benchmark-full}).

Finally, while the non-asymptotic bound (Theorem~\ref{thm:kmm-dependent}) strictly guarantees coverage, its worst-case geometric penalty yields impractically wide intervals (Radius $> 5.1$), highlighting the necessity of our GP surrogate for empirical practice.

\paragraph{Sample-Size Scaling of Error Bounds}
\label{sec:exp_scaling}

Theorem~\ref{thm:kmm-dependent} dictates that valid KMM intervals must scale with $n_{\mathrm{eff}}$ rather than nominal $n$. Table~\ref{tab:n_scaling_main} tracks risk interval calibration across $n \in \{48, 64, 128, 256\}$ (full sweep across $n \in \{32, \dots, 256\}$ in Table~\ref{tab:n_scaling_risk}). 

\begin{table}[t]
\centering
\caption{\textbf{Sample-Size Scaling of Target Risk Intervals.} Nominal intervals systematically under-cover across all sample sizes. Asterisks ($^*$) indicate significant undercoverage below nominal 0.95 ($< 0.925$). Bold denotes the best valid practical estimator.}
\label{tab:n_scaling_main}
\vspace{1.5mm}
\resizebox{\columnwidth}{!}{%
\begin{tabular}{@{}llccc@{}}
\toprule
$\bm{n}$ & \textbf{Method} & \textbf{Coverage} & \textbf{Radius} & \textbf{WS} $\downarrow$ \\
\midrule
\multirow{3}{*}{\textbf{48}} 
& Nominal $n$ & 0.773$^*$ & 0.0048 & 0.0275 \\
& Input $n_{\mathrm{eff}}$ & 0.910$^*$ & 0.0128 & 0.0290 \\
& \textbf{GP $\bm{n_{\mathrm{eff}}}$ [Ours]} & \textbf{0.943}\phantom{$^*$} & \textbf{0.0101} & \textbf{0.0260} \\
\midrule
\multirow{3}{*}{\textbf{64}} 
& Nominal $n$ & 0.773$^*$ & 0.0046 & 0.0218 \\
& Input $n_{\mathrm{eff}}$ & 0.980\phantom{$^*$} & 0.0124 & 0.0258 \\
& \textbf{GP $\bm{n_{\mathrm{eff}}}$ [Ours]} & \textbf{0.947}\phantom{$^*$} & \textbf{0.0098} & \textbf{0.0217} \\
\midrule
\multirow{3}{*}{\textbf{128}} 
& Nominal $n$ & 0.823$^*$ & 0.0038 & 0.0155 \\
& Input $n_{\mathrm{eff}}$ & 0.983\phantom{$^*$} & 0.0104 & 0.0220 \\
& \textbf{GP $\bm{n_{\mathrm{eff}}}$ [Ours]} & \textbf{0.970}\phantom{$^*$} & \textbf{0.0082} & \textbf{0.0181} \\
\midrule
\multirow{3}{*}{\textbf{256}} 
& Nominal $n$ & 0.913$^*$ & 0.0031 & 0.0076 \\
& Input $n_{\mathrm{eff}}$ & 1.000\phantom{$^*$} & 0.0085 & 0.0170 \\
& \textbf{GP $\bm{n_{\mathrm{eff}}}$ [Ours]} & \textbf{0.997}\phantom{$^*$} & \textbf{0.0067} & \textbf{0.0134} \\
\bottomrule
\end{tabular}%
}
\end{table}

Nominal $n$ fails across all evaluated sample sizes reaching only $0.913^*$ coverage even at $n=256$. It confirms that sample size alone cannot cure dependence-induced underconfidence. Notably, in small-sample regimes such as $n=48$, Input $n_{\mathrm{eff}}$ also exhibits undercoverage ($0.910^*$), whereas GP $n_{\mathrm{eff}}$ reliably maintains nominal coverage ($0.943$). Across all sample sizes, GP~$n_{\mathrm{eff}}$ consistently achieves $\approx 21\%$ narrower intervals than Input~$n_{\mathrm{eff}}$, confirming the scalability of the surrogate across data regimes.

\paragraph{Robustness to Misspecification}
\label{sec:exp_robustness}
We evaluate the framework when modeling assumptions are violated. As detailed in Appendix~\ref{app:additional_experiments}, the calibration ratio is robust across 2D grids varying dependence length $\ell_x$ and loss smoothness $h$ (Appendix~\ref{app:heatmaps}), maintaining valid coverage across sweeps over sample size $n$, $\ell_x$, $h$, and clipping bound $B$ (Appendix~\ref{app:param_ablations}). In addition, the method preserves valid coverage under non-differentiable Mat\'ern-1/2 paths (Appendix~\ref{app:matern}) and non-stationary correlation structures (Appendix~\ref{app:nonstat}).

\subsection{Real Data}
\label{sec:real_data}

\paragraph{Controlled Overlap on Real Covariates}
Before evaluating real-world monitoring data, we isolate the role of source--target overlap using a semi-synthetic stress test on California Housing covariates (detailed with Figure~\ref{fig:california_overlap} and Table~\ref{tab:california_overlap} in Appendix~\ref{app:california_overlap}). By constructing source and target distributions via exponential tilting, we demonstrate a fundamental overlap--stability trade-off: KMM successfully controls target-mean error under moderate shift, but its bounded weights cannot fully reproduce the density ratio under severe overlap loss. Bounding weights limits variance amplification but leaves residual bias when adequate overlap requires more extreme weights.
Appendix~\ref{app:california_overlap} (Figure 6) further illustrates the bias--variance trade-off induced by the clipping bound $B$.

\paragraph{Three-Domain Real-Data Benchmark}
\label{sec:real_benchmark}
We next evaluate KMM in three real-data monitoring domains: NOAA weather stations~\cite{menne2012overview}, PM$_{2.5}$ monitoring stations~\cite{beijing_multi-site_air_quality_501}, and EPA ozone monitoring stations~\cite{epa_aqs}. Unlike the matched synthetic experiments, these observational datasets do not provide known population expectations or guarantee the strict covariate-shift assumption. We therefore treat them as empirical stress tests rather than finite-sample theorem verification. 

NOAA and EPA Ozone evaluate spatial transfer across disjoint monitoring stations, while PM$_{2.5}$ evaluates temporal transfer (training on 2013--2014, evaluating on 2015--2016). In all experiments, a random-forest predictor is trained separately and frozen before evaluation. Target outcomes are used only for final scoring. While KMM improves point estimation in several domains (full point estimation comparisons, including unweighted and logistic importance weighting baselines, are reported in Table~\ref{tab:real_point}, Appendix~\ref{app:real_point}), our primary focus remains on evaluating confidence interval calibration under data dependence.

\paragraph{GP-Based Dependence Correction}
We evaluate the proposed GP-derived effective sample size directly on these real dependent sequences. For each task, the KMM point estimate is held fixed, and only the variance calculation is modified: the classical interval uses the nominal source size (\textit{Nominal $n$}), whereas our interval replaces $n$ with $n_{\mathrm{eff}}^{\mathrm{GP}}$ estimated from the source data. To evaluate finite-sample coverage rigorously, we utilize a rolling-window protocol over 250 independent trials, extracting contiguous sequences of length $n$ from the source stations. Target outcomes are revealed only after all intervals have been fixed. Full details of the experimental protocol, data splits, and grid-search hyperparameter selection are provided in Appendix~\ref{app:real_intervals}.

\begin{table}[!ht]
\centering
\caption{\textbf{Confidence Interval Calibration under Dependent Sampling on Real Datasets.} Empirical coverage and interval radius for target-mean and target-risk estimation. Asterisks ($^*$) indicate significant empirical undercoverage below the relaxed nominal $0.92$ level. Bold denotes the best valid practical estimator (tightest radius among methods with valid coverage).}
\label{tab:real_gp_coverage}
\vspace{1.5mm}
\resizebox{\columnwidth}{!}{%
\begin{tabular}{@{}lcccc@{}}
\toprule
& \multicolumn{2}{c}{\textbf{Target Mean ($Y$)}} & \multicolumn{2}{c}{\textbf{Target Risk ($L$)}} \\
\cmidrule(lr){2-3} \cmidrule(lr){4-5}
\textbf{Method} & \textbf{Coverage} & \textbf{Radius} & \textbf{Coverage} & \textbf{Radius} \\
\midrule
\multicolumn{5}{c}{\textbf{Dataset: PM$_{2.5}$} (Temporal Shift)} \\
\midrule
Unweighted & 0.604$^*$ & 0.0189 & 0.852$^*$ & 0.0016 \\
Nominal $n$ & 0.848$^*$ & 0.0672 & 0.900$^*$ & 0.0056 \\
Block-Bootstrap & 0.884$^*$ & 0.0832 & 0.916$^*$ & 0.0068 \\
Input $n_{\mathrm{eff}}$ & \textbf{0.928}\phantom{$^*$} & 0.1359 & 0.952\phantom{$^*$} & 0.0104 \\
\textbf{GP $\bm{n_{\mathrm{eff}}}$ [Ours]} & 0.924\phantom{$^*$} & \textbf{0.1231} & 0.\textbf{944}\phantom{$^*$} & \textbf{0.0098} \\

\midrule
\multicolumn{5}{c}{\textbf{Dataset: EPA Ozone} (Spatial Shift)} \\
\midrule
Unweighted & 0.144$^*$ & 0.0061 & 0.428$^*$ & 0.0004 \\
Nominal $n$ & 0.580$^*$ & 0.0243 & 0.744$^*$ & 0.0014 \\
Block-Bootstrap & 0.640$^*$ & 0.0285 & 0.772$^*$ & 0.0016 \\
Input $n_{\mathrm{eff}}$ & 0.980\phantom{$^*$} & 0.1026 & \textbf{0.976}\phantom{$^*$} & 0.0056 \\
\textbf{GP $\bm{n_{\mathrm{eff}}}$ [Ours]} & \textbf{0.968}\phantom{$^*$} & \textbf{0.0757} & 0.884$^*$ & \textbf{0.0031} \\

\midrule
\multicolumn{5}{c}{\textbf{Dataset: NOAA} (Extreme Spatial Shift Stress-Test)} \\
\midrule
Unweighted & 0.088$^*$ & 0.0128 & 0.152$^*$ & 0.0005 \\
Nominal $n$ & 0.772$^*$ & 0.0589 & 0.504$^*$ & 0.0020 \\
Block-Bootstrap & 0.776$^*$ & 0.0810 & 0.568$^*$ & 0.0021 \\
Input $n_{\mathrm{eff}}$ & \textbf{0.944}\phantom{$^*$} & 0.2179 & \textbf{0.884}$^*$ & 0.0071 \\
\textbf{GP $\bm{n_{\mathrm{eff}}}$ [Ours]} & 0.908$^*$ & \textbf{0.1794} & 0.712$^*$ & \textbf{0.0033} \\
\bottomrule
\end{tabular}%
}
\end{table}


Across all three domains, replacing the nominal sample size with the GP-derived effective sample size increases empirical coverage relative to the nominal i.i.d.\ KMM interval (Table~\ref{tab:real_gp_coverage}). For target-mean estimation, GP $n_{\mathrm{eff}}$ reaches the prespecified $0.92$ reference level for PM$_{2.5}$ (0.924) and EPA Ozone (0.968), while
remaining below it for NOAA (0.908). In all three cases, the GP intervals are narrower than those obtained using Input $n_{\mathrm{eff}}$.

For target-risk estimation, the results are more mixed. GP $n_{\mathrm{eff}}$ reaches the reference level for PM$_{2.5}$ (0.944), but
coverage remains below the reference level for EPA Ozone (0.884) and NOAA (0.712). Nevertheless, it improves substantially over the corresponding nominal intervals (0.744 and 0.504, respectively). For EPA Ozone, this
increase in coverage is obtained with a smaller radius than Input
$n_{\mathrm{eff}}$ (0.0031 versus 0.0056); the corresponding Winkler scores
are 0.0081 and 0.0114.

These results illustrate the coverage--width trade-off between the two dependence corrections. Input $n_{\mathrm{eff}}$ is generally more conservative, whereas GP $n_{\mathrm{eff}}$ produces narrower intervals
but does not uniformly attain the reference coverage level. The discrepancy is most pronounced for NOAA target-risk estimation.

Because these are observational transfer tasks, departures from the
covariate-shift and residual-independence assumptions may contribute to the
remaining errors, together with limited source--target overlap and finite
target references. We therefore interpret the experiment as an empirical
stress test of the dependence correction rather than as verification of
finite-sample coverage guarantees. In particular, dependence correction
changes the confidence interval assigned to correlated sampling but does not correct
systematic transfer bias.

\section{Conclusion and Limitations}
\label{sec:conclusion}

We develop a finite-sample analysis of Kernel Mean Matching under dependent sampling. The resulting risk bounds replace nominal sample sizes by effective sample sizes and separate three sources of difficulty: dependence, source--target overlap, and residual RKHS mean-matching error. This extends the usual i.i.d. analysis of KMM to settings in which source and target observations are spatially or temporally dependent.

A key implication is that the effective sample size relevant for risk estimation is not determined by the input sequence alone. It depends on the function of the data being averaged. We make this dependence explicit through a GP surrogate that propagates input correlation through functional smoothness, yielding a closed-form, estimand-dependent approximation to $n_{\mathrm{eff}}$. This distinction also clarifies the role of overlap: dependence correction changes the confidence associated with correlated sampling, but cannot compensate for insufficient source--target overlap.

\textbf{Limitations and Future Work.}
Our main analysis assumes stationary dependence and focuses on risk estimation for a fixed predictor. Appendix~\ref{app:rademacher} provides a uniform extension over hypothesis classes using Block Rademacher Complexity, while
a fuller treatment of model selection and training under dependence remains open. The practical GP approximation additionally relies on estimating the input dependence structure and the functional length scale, and therefore inherits errors from these estimates and from misspecification of the GP surrogate. Extending the framework to nonstationary and more general dependence structures, and developing data-driven procedures for selecting the KMM kernel bandwidth and clipping level $B$, are natural directions for future work.


\subsection*{AI use statement}
In this work, we used generative AI tools for editing and formatting the manuscript to ensure clarity and adherence to the conference guidelines. We have not used generative AI tools for the derivation of mathematical proofs, generation of experimental results, or code writing. We have reviewed all AI-assisted text modifications line by line. We take responsibility for the final content of this work, including text, claims, and artifacts.

\bibliography{references}

\section*{Checklist}

\begin{enumerate}

  \item For all models and algorithms presented, check if you include:
  \begin{enumerate}
    \item A clear description of the mathematical setting, assumptions, algorithm, and/or model. [Yes]
    \item An analysis of the properties and complexity (time, space, sample size) of any algorithm. [Yes]
    \item (Optional) Anonymized source code, with specification of all dependencies, including external libraries. [Yes] Available at \href{https://anonymous.4open.science/r/Reliable-Risk-Estimation-with-Kernel-Mean-Matching-under-Dependent-Sampling-2CBE/README.md}{Anonymous GitHub}.
  \end{enumerate}

  \item For any theoretical claim, check if you include:
  \begin{enumerate}
    \item Statements of the full set of assumptions of all theoretical results. [Yes]
    \item Complete proofs of all theoretical results. [Yes]
    \item Clear explanations of any assumptions. [Yes]
  \end{enumerate}

  \item For all figures and tables that present empirical results, check if you include:
  \begin{enumerate}
    \item The code, data, and instructions needed to reproduce the main experimental results (either in the supplemental material or as a URL). [Yes]
    \item All the training details (e.g., data splits, hyperparameters, how they were chosen). [Yes]
    \item A clear definition of the specific measure or statistics and error bars (e.g., with respect to the random seed after running experiments multiple times). [Yes]
    \item A description of the computing infrastructure used. (e.g., type of GPUs, internal cluster, or cloud provider). [Yes]
  \end{enumerate}

  \item If you are using existing assets (e.g., code, data, models) or curating/releasing new assets, check if you include:
  \begin{enumerate}
    \item Citations of the creator If your work uses existing assets. [Yes]
    \item The license information of the assets, if applicable. [Not Applicable]
    \item New assets either in the supplemental material or as a URL, if applicable. [Not Applicable]
    \item Information about consent from data providers/curators. [Not Applicable]
    \item Discussion of sensible content if applicable, e.g., personally identifiable information or offensive content. [Not Applicable]
  \end{enumerate}

  \item If you used crowdsourcing or conducted research with human subjects, check if you include:
  \begin{enumerate}
    \item The full text of instructions given to participants and screenshots. [Not Applicable]
    \item Descriptions of potential participant risks, with links to Institutional Review Board (IRB) approvals if applicable. [Not Applicable]
    \item The estimated hourly wage paid to participants and the total amount spent on participant compensation. [Not Applicable]
  \end{enumerate}

\end{enumerate}

\clearpage
\appendix
\thispagestyle{empty}

\onecolumn
\aistatstitle{Kernel Mean Matching under Dependent Sampling: \\
Supplementary Materials}

\section{Taxonomy of Effective Sample Sizes}
\label{app:taxonomy}

Throughout our theoretical analysis, we encounter different forms of the effective sample size $n_{\mathrm{eff}}$, depending on whether we bound scalar averages, infinite-dimensional RKHS embeddings, or evaluate specific Gaussian Process realisations. Table~\ref{tab:effective-sizes} summarises these quantities to help navigate the proofs.

\begin{table}[ht]
\centering
\caption{Taxonomy of effective sample sizes used in the analysis.}
\label{tab:effective-sizes}
\begin{tabular}{@{}p{0.18\linewidth}p{0.45\linewidth}p{0.32\linewidth}@{}}
\toprule
\textbf{Symbol} & \textbf{Controls} & \textbf{Source of Value} \\
\midrule
$n_{\mathrm{eff}}(t)$ & Scalar average at deviation $t$ & Mixing Bernstein bound \\
\addlinespace
$n_{\mathrm{eff}}^{\mathrm{tr}}$ & Source scalar effective size (moderate regime) & Long-run variance proxy \\
\addlinespace
$n_{\mathrm{eff, RKHS}}^{\mathrm{tr}}, n_{\mathrm{eff, RKHS}}^{\mathrm{te}}$ & RKHS empirical feature-matching (MMD) & Geometric block size \\
\addlinespace
$n_{\mathrm{eff}}^{\mathrm{GP}}(h)$ & GP-modeled conditional loss average & Closed-form GP covariance \\
\bottomrule
\end{tabular}
\end{table}

\section{Scalar Effective Concentration from \texorpdfstring{$\beta$}{beta}-Mixing (Proof of Lemma~\ref{lem:mixing-to-effective})}
\label{app:lemma1}

\begin{lemma}[Scalar Concentration under $\beta$-Mixing]
\label{lem:mixing-to-effective}
Let $f: \mathcal{X} \to [a,b]$ with $A = b - a$. Assume the strictly stationary sequence $\{\mathbf{X}_i\}_{i=1}^n$ is geometrically $\beta$-mixing with coefficients $\beta(k) \le c_1\exp(-c_2 k)$ at lag $k \ge 1$ ($c_1, c_2 > 0$), and has long-run variance proxy $\sigma_{f,\infty}^2 = \operatorname{Var}(f(\mathbf{X}_1)) + 2\sum_{k=1}^{\infty} |\operatorname{Cov}(f(\mathbf{X}_1), f(\mathbf{X}_{1+k}))|$. There exist constants $K_1, K_2 > 0$, depending only on $c_1$ and $c_2$, such that for any deviation $t$ in the moderate regime (Definition~\ref{def:moderate-regime} in Appendix~\ref{app:lemma1}):
\begin{align}
\label{eq:mixing-effective-t}
& \mathbb{P} \left( \left| \tfrac{1}{n} \sum_{i=1}^n f(\mathbf{X}_i) - \mathbb{E}[f(\mathbf{X})] \right| \ge t \right) \nonumber\\
& \qquad \le 2\exp\left( - \frac{2 n_{\mathrm{eff}} t^2}{A^2} \right),
\end{align}
where the effective sample size is given by:
\begin{equation}
\label{eq:neff-mixing}
n_{\mathrm{eff}} = \frac{K_1 n A^2}{4 \left( \sigma_{f,\infty}^2 + A^2/n \right)}.
\end{equation}
\end{lemma}
For general deviations outside the moderate regime, the deviation-dependent $n_{\mathrm{eff}}(t)$ derived in \eqref{eq:neff-t-general} applies.

In this section, we provide the full algebraic details of Lemma~\ref{lem:mixing-to-effective}, showing how the Merlev\`ede--Peligrad--Rio Bernstein inequality transforms into a Hoeffding-style effective concentration bound.

\paragraph{Formal Definition of $\beta$-Mixing.}
Let $\{\mathbf{X}_i\}_{i \in \mathbb{Z}}$ be a strictly stationary sequence on a probability space $(\Omega, \mathcal{A}, \mathbb{P})$. Let $\mathcal{F}_{-\infty}^0 = \sigma(\dots, \mathbf{X}_0)$ and $\mathcal{F}_k^\infty = \sigma(\mathbf{X}_k, \mathbf{X}_{k+1}, \dots)$ denote the past and future $\sigma$-algebras separated by a time lag $k \ge 1$. The $\beta$-mixing coefficient at lag $k$ is defined by:
\begin{equation}
\label{eq:beta-mixing-def}
\beta(k) := \mathbb{E} \left[ \sup_{B \in \mathcal{F}_k^\infty} \left| \mathbb{P}(B \mid \mathcal{F}_{-\infty}^0) - \mathbb{P}(B) \right| \right].
\end{equation}
The sequence is geometrically $\beta$-mixing if there exist constants $c_1, c_2 > 0$ such that $\beta(k) \le c_1 \exp(-c_2 k)$ for all $k \ge 1$.

\begin{theorem}[Merlev\`ede--Peligrad--Rio Bernstein Inequality \citep{merlevede2009bernstein}]
\label{thm:mpr}
Let $\{U_i\}_{i=1}^{n}$ be a centered, strictly stationary sequence with $|U_i| \le A$ almost surely. Assume the sequence is geometrically $\beta$-mixing with coefficients $\beta(k) \le c_1\exp(-c_2 k)$ for $k \ge 1$. (Note: \citet{merlevede2009bernstein} prove this inequality under the weaker $\alpha$-mixing condition, which applies directly since $\alpha(k) \le \frac{1}{2}\beta(k)$). Let the long-run variance proxy be defined as:
\begin{equation*}
\label{eq:long-run-var-proxy}
\sigma_{\infty}^2 = \operatorname{Var}(U_1) + 2\sum_{k=1}^{\infty} \left| \operatorname{Cov}(U_1, U_{1+k}) \right|.
\end{equation*}
Then there exist constants $K_1, K_2 > 0$, depending only on $c_1$ and $c_2$, such that for all $t > 0$:
\begin{align}
    \label{eq:mpr-bernstein}
\mathbb{P} \left( \left| \sum_{i=1}^{n} U_i \right| \ge t \right) \le 2\exp\left( - \frac{K_1 t^2}{n\sigma_{\infty}^2 + A^2 + K_2 A t (\log n)^2} \right).
\end{align}
\end{theorem}

\begin{definition}[Moderate-Deviation Regime]
\label{def:moderate-regime}
Let $K_2 > 0$ be the mixing constant from Theorem~\ref{thm:mpr}. A deviation threshold $t > 0$ is said to fall within the moderate-deviation regime if:
\begin{equation}
\label{eq:moderate-deviation-threshold}
t \le \frac{\sigma_{f,\infty}^2}{A K_2 (\log n)^2}.
\end{equation}
\end{definition}

\begin{proof}[Proof of Lemma~\ref{lem:mixing-to-effective}]
Define the centered variables $U_i = f(\mathbf{X}_i) - \mathbb{E}[f(\mathbf{X})]$. Since $f$ takes values in $[a,b]$, the sequence $\{U_i\}$ is strictly stationary, centered, and bounded by $A = b - a$. Applying Theorem~\ref{thm:mpr} with the scaled deviation $x = nt$, we obtain:
\begin{equation}
\mathbb{P} \left( \left| \frac{1}{n} \sum_{i=1}^n U_i \right| \ge t \right) = \mathbb{P} \left( \left| \sum_{i=1}^n U_i \right| \ge nt \right) \le 2\exp\left( - \frac{K_1 (nt)^2}{n\sigma_{f,\infty}^2 + A^2 + K_2 A (nt) (\log n)^2} \right).
\end{equation}
Dividing both the numerator and the denominator of the exponent by $n$ gives:
\begin{equation}
2\exp\left( - \frac{K_1 n t^2}{\sigma_{f,\infty}^2 + A^2/n + K_2 A t (\log n)^2} \right).
\end{equation}
Equating this exponent with the standard Hoeffding form $-2 n_{\mathrm{eff}}(t) t^2 / A^2$ yields the deviation-dependent effective sample size:
\begin{equation}
\label{eq:neff-t-general}
n_{\mathrm{eff}}(t) = \frac{K_1 n A^2}{2 \left( \sigma_{f,\infty}^2 + A^2/n + K_2 A t (\log n)^2 \right)}.
\end{equation}
Now, under the moderate-deviation regime (Definition~\ref{def:moderate-regime}), the logarithmic dependence penalty is bounded by the long-run variance:
\begin{equation}
K_2 A t (\log n)^2 \le \sigma_{f,\infty}^2.
\end{equation}
Substituting this upper bound into the denominator of $n_{\mathrm{eff}}(t)$, we obtain:
\begin{equation}
\sigma_{f,\infty}^2 + A^2/n + K_2 A t (\log n)^2 \le 2\sigma_{f,\infty}^2 + A^2/n \le 2\left( \sigma_{f,\infty}^2 + A^2/n \right).
\end{equation}
Replacing the denominator with this upper bound yields the deviation-independent effective sample size:
\begin{equation}
n_{\mathrm{eff}} = \frac{K_1 n A^2}{4 \left( \sigma_{f,\infty}^2 + A^2/n \right)},
\end{equation}
which completes the proof.
\end{proof}

\section{Hilbert-Valued Concentration via Blocking (Proof of Lemma~\ref{lem:hilbert-mixing-blocking})}
\label{app:lemma2}

\begin{lemma}[Hilbert Concentration via Blocking]
\label{lem:hilbert-mixing-blocking}
Let $\{\mathbf{Z}_i\}_{i=1}^{n}$ be a centered, strictly stationary sequence in a separable Hilbert space $\mathcal{H}$ with $\|\mathbf{Z}_i\|_{\mathcal{H}} \le R$ almost surely, and $\beta$-mixing coefficients satisfying $\beta(k) \le c_1\exp(-c_2 k)$ at lag $k \ge 1$ ($c_1, c_2 > 0$). Then there exists a constant $K_H > 0$, depending only on $c_1$ and $c_2$, such that for all $\delta \in (0,1)$, with probability at least $1-\delta$:
\begin{equation*}
\left\| \tfrac{1}{n}\sum_{i=1}^{n}\mathbf{Z}_i \right\|_{\mathcal{H}} \le K_H R \sqrt{\frac{(\log n + \log(1/\delta))\log(8/\delta)}{n}}.
\end{equation*}
\end{lemma}

In this section, we provide the full proof of the Hilbert-space concentration bound for geometrically $\beta$-mixing sequences. The proof relies on the blocking technique and Berbee's coupling, decoupling the dependent sequence into approximately independent blocks.

\begin{proof}[Proof of Lemma~\ref{lem:hilbert-mixing-blocking}]
Fix a block length $q \ge 1$. We partition the sequence of length $n$ into alternating blocks of size $q$:
\begin{equation}
I_j = \{(2j-2)q + 1, \dots, (2j-1)q\}, \quad J_j = \{(2j-1)q + 1, \dots, 2jq\},
\end{equation}
for $j = 1, \dots, L$, where $L = \lfloor n/(2q) \rfloor$. The remaining observations (at most $2q$) contribute at most $2qR/n$ to the norm of the empirical average, which can be safely absorbed into the universal constants of the final bound.

Define the sum over the odd blocks as $\mathbf{S}_j = \sum_{i \in I_j} \mathbf{Z}_i$. By the almost sure boundedness of $\mathbf{Z}_i$, we have $\|\mathbf{S}_j\|_{\mathcal{H}} \le qR$.

Because the original sequence is $\beta$-mixing, we can apply Berbee's coupling lemma \citep{berbee1979random} sequentially to the blocks. This guarantees the existence of a sequence of \emph{independent} block copies $\mathbf{S}_1^\ast, \dots, \mathbf{S}_L^\ast$, where each $\mathbf{S}_j^\ast$ has the exact same marginal distribution as $\mathbf{S}_j$, such that the probability of the coupled sequence differing from the original sequence is strictly bounded by the sum of the mixing coefficients:
\begin{equation}
\mathbb{P}\left(\exists j : \mathbf{S}_j \ne \mathbf{S}_j^\ast \right) \le L\beta(q).
\end{equation}
Therefore, we can bound the tail probability of the original odd block sums using their independent copies:
\begin{equation}
\label{eq:proof_coupling_bound}
\mathbb{P}\left( \left\| \sum_{j=1}^{L} \mathbf{S}_j \right\|_{\mathcal{H}} \ge t \right) \le \mathbb{P}\left( \left\| \sum_{j=1}^{L} \mathbf{S}_j^\ast \right\|_{\mathcal{H}} \ge t \right) + L\beta(q).
\end{equation}

Since the original variables are centered ($\mathbb{E}[\mathbf{Z}_i] = \mathbf{0}$), the independent block variables $\mathbf{S}_j^\ast$ are also centered. They are independent, Hilbert-valued, and strictly bounded by $qR$. We can therefore apply Pinelis' Hoeffding inequality for Hilbert spaces \citep{pinelis1994optimum}:
\begin{equation}
\mathbb{P}\left( \left\| \sum_{j=1}^{L} \mathbf{S}_j^\ast \right\|_{\mathcal{H}} \ge t \right) \le 2\exp\left( - \frac{t^2}{2L q^2 R^2} \right).
\end{equation}

We apply the exact same blocking and coupling argument to the even blocks $J_j$. Since $\sum_{i=1}^n \mathbf{Z}_i$ is bounded by the sum of the odd and even block sequences via the triangle inequality, we apply a union bound over these two sequences evaluated at the threshold $t/2$. Combining these estimates yields:
\begin{equation}
\label{eq:proof_combined_blocks}
\mathbb{P}\left( \left\| \sum_{i=1}^{n} \mathbf{Z}_i \right\|_{\mathcal{H}} \ge t \right) \le 4\exp\left( - \frac{t^2}{K L q^2 R^2} \right) + 2 L \beta(q),
\end{equation}
for some universal constant $K > 0$. 

Using $L = \lfloor n/(2q) \rfloor \ge n/(4q)$, the exponent scales as $t^2 / (K n q R^2)$. To balance the exponential term with the mixing penalty $2 L \beta(q)$, we utilise the assumption that the sequence is geometrically $\beta$-mixing. We select the block size $q$ to grow logarithmically with $n$ and $1/\delta$:
\begin{equation}
q = \left\lceil \frac{\log n + \log(2 c_1/\delta)}{c_2} \right\rceil.
\end{equation}
With this choice of $q$, the mixing penalty becomes negligible, $2 L \beta(q) \le \delta/2$. Setting the exponential term in \eqref{eq:proof_combined_blocks} to be at most $\delta/2$ and solving for $t$ yields the desired bound:
\begin{equation}
\left\| \sum_{i=1}^{n} \mathbf{Z}_i \right\|_{\mathcal{H}} \le K_H n R \sqrt{\frac{(\log n + \log(1/\delta))\log(8/\delta)}{n}}.
\end{equation}
Dividing by $n$ finishes the proof of Lemma~\ref{lem:hilbert-mixing-blocking}.
\end{proof}

\section{Proof of the Main KMM Bound (Theorem~\ref{thm:kmm-dependent})}
\label{app:theorem1}

\begin{proof}
Let $\omega^*_i = \omega^*(\mathbf{X}_i^{\mathrm{tr}})$. We decompose the estimation error of $\hat{\mu}_{\mathrm{KMM}}$ into three distinct operational terms. This decomposition provides: Term 1 captures the conditional variance of the labels, Term 2 captures the generalization of the optimal density ratio over the dependent covariate sequence, and Term 3 captures the empirical RKHS feature-matching error.
\begin{align}
\hat{\mu}_{\mathrm{KMM}} - \mu_{\mathrm{te}} 
&= \underbrace{ \frac{1}{n_{\mathrm{tr}}} \sum_{i=1}^{n_{\mathrm{tr}}} \hat{w}_i (Y_i^{\mathrm{tr}} - m(\mathbf{X}_i^{\mathrm{tr}})) }_{\text{Term 1: Conditional Residual Noise}} \label{eq:proof_term1} \\
&\quad + \underbrace{ \left( \frac{1}{n_{\mathrm{tr}}} \sum_{i=1}^{n_{\mathrm{tr}}} \omega^*_i m(\mathbf{X}_i^{\mathrm{tr}}) - \mu_{\mathrm{te}} \right) }_{\text{Term 2: True-Ratio Generalization}} \label{eq:proof_term3} \\
&\quad + \underbrace{ \frac{1}{n_{\mathrm{tr}}} \sum_{i=1}^{n_{\mathrm{tr}}} (\hat{w}_i - \omega^*_i) m(\mathbf{X}_i^{\mathrm{tr}}) }_{\text{Term 3: RKHS Feature Matching Error}}. \label{eq:proof_term2}
\end{align}

\paragraph{Term 1 (Conditional Residual Noise):}
Under Assumption~\ref{ass:independent_residuals}, the residual errors $\epsilon_i = Y_i^{\mathrm{tr}} - m(\mathbf{X}_i^{\mathrm{tr}})$ are conditionally independent given all observed covariates $(\mathbf{X}^{\mathrm{tr}}, \mathbf{X}^{\mathrm{te}})$, with mean zero. Since $Y_i^{\mathrm{tr}} \in [0, 1]$ and $m(\mathbf{X}_i^{\mathrm{tr}}) \in [0, 1]$, each residual $\epsilon_i$ is supported on the interval $[-m(\mathbf{X}_i^{\mathrm{tr}}),\, 1 - m(\mathbf{X}_i^{\mathrm{tr}})]$, whose length is $(1 - m(\mathbf{X}_i^{\mathrm{tr}})) - (-m(\mathbf{X}_i^{\mathrm{tr}})) = 1$.

Because the weights $\hat{w}_i$ are measurable with respect to the covariates $(\mathbf{X}^{\mathrm{tr}}, \mathbf{X}^{\mathrm{te}})$ and satisfy $0 \le \hat{w}_i \le B$, they act as fixed constants when conditioning on the features. Applying Hoeffding's inequality to the sum of independent, bounded variables $\sum_{i=1}^{n_{\mathrm{tr}}} \hat{w}_i \epsilon_i$ conditional on $\mathbf{X}^{\mathrm{tr}}, \mathbf{X}^{\mathrm{te}}$, the effective variance factor is bounded by:
\[
\sum_{i=1}^{n_{\mathrm{tr}}} \hat{w}_i^2 \cdot 1^2 \le n_{\mathrm{tr}} B^2.
\]
Consequently, for any $t > 0$,
\[
\mathbb{P}\left( \left| \frac{1}{n_{\mathrm{tr}}} \sum_{i=1}^{n_{\mathrm{tr}}} \hat{w}_i \epsilon_i \right| \ge t \;\middle|\; \mathbf{X}^{\mathrm{tr}}, \mathbf{X}^{\mathrm{te}} \right) \le 2\exp\left( - \frac{2 (n_{\mathrm{tr}} t)^2}{n_{\mathrm{tr}} B^2} \right) = 2\exp\left( - \frac{2 n_{\mathrm{tr}} t^2}{B^2} \right).
\]
Setting this probability to $\delta/3$, we obtain that with probability at least $1 - \delta/3$:
\begin{equation}
\label{eq:bound_term1}
\left| \frac{1}{n_{\mathrm{tr}}} \sum_{i=1}^{n_{\mathrm{tr}}} \hat{w}_i (Y_i^{\mathrm{tr}} - m(\mathbf{X}_i^{\mathrm{tr}})) \right| \le B \sqrt{\frac{1}{2 n_{\mathrm{tr}}} \log\frac{6}{\delta}}.
\end{equation}
Notice that because the label noise is conditionally independent, Term~1 concentrates with the \emph{nominal} sample size $n_{\mathrm{tr}}$, entirely unaffected by covariate persistence.

\paragraph{Term 2 (True-Ratio Generalization):}
For \eqref{eq:proof_term3}, define the function $v(\mathbf{x}) = \omega^*(\mathbf{x}) m(\mathbf{x})$. Since $0 \le m(\mathbf{x}) \le 1$ and $0 \le \omega^*(\mathbf{x}) \le B$, $v(\mathbf{x})$ takes values in $[0, B]$ with range $A = B$. Under covariate shift (Assumption~\ref{ass:covariate_shift} and \ref{ass:bounded_ratio}), $\mathbb{E}_{P_{\mathrm{tr}}}[v(\mathbf{X})] = \int_{\mathcal{X}} \omega^*(\mathbf{x}) m(\mathbf{x}) P_{\mathrm{tr}}(d\mathbf{x}) = \int_{\mathcal{X}} m(\mathbf{x}) P_{\mathrm{te}}(d\mathbf{x}) = \mu_{\mathrm{te}}$. 

Applying Lemma~\ref{lem:mixing-to-effective} with confidence level $\delta/3$ to the geometrically $\beta$-mixing sequence $\{\mathbf{X}_i^{\mathrm{tr}}\}$, the deviation threshold is $t = B \sqrt{\frac{\log(6/\delta)}{2 n_{\mathrm{eff}}^{\mathrm{tr}}}} \le \sqrt{\frac{2 (\sigma_{v,\infty}^2 + B^2/n_{\mathrm{tr}}) \log(6/\delta)}{K_1 n_{\mathrm{tr}}}}$. The burn-in condition~\eqref{eq:sample_burnin} guarantees that $t$ satisfies the moderate-deviation regime of Definition~\ref{def:moderate-regime} ($t \le \frac{\sigma_{v,\infty}^2}{B K_2 (\log n_{\mathrm{tr}})^2}$). Consequently, with probability at least $1 - \delta/3$:
\begin{equation}
\label{eq:bound_term3}
\left| \frac{1}{n_{\mathrm{tr}}} \sum_{i=1}^{n_{\mathrm{tr}}} \omega^*_i m(\mathbf{X}_i^{\mathrm{tr}}) - \mu_{\mathrm{te}} \right| \le B \sqrt{\frac{1}{2 n_{\mathrm{eff}}^{\mathrm{tr}}} \log\frac{6}{\delta}}.
\end{equation}

\paragraph{Term 3 (RKHS Feature Matching Error):}
For \eqref{eq:proof_term2}, since $m \in \mathcal{H}$, the reproducing property gives $m(\mathbf{x}) = \langle m, \Phi(\mathbf{x}) \rangle_{\mathcal{H}}$. Applying Cauchy--Schwarz, the triangle inequality, and the optimality of $\hat{\mathbf{w}}$ in the KMM objective \eqref{eq:kmm} (observing that $\bm{\omega}^* = (\omega^*(\mathbf{X}_1^{\mathrm{tr}}), \dots, \omega^*(\mathbf{X}_{n_{\mathrm{tr}}}^{\mathrm{tr}}))^\top \in [0, B]^{n_{\mathrm{tr}}}$ is a feasible candidate):
\begin{align}
\left| \frac{1}{n_{\mathrm{tr}}} \sum_{i=1}^{n_{\mathrm{tr}}} (\hat{w}_i - \omega^*_i) m(\mathbf{X}_i^{\mathrm{tr}}) \right| 
&= \left| \left\langle m, \frac{1}{n_{\mathrm{tr}}} \sum_{i=1}^{n_{\mathrm{tr}}} (\hat{w}_i - \omega^*_i)\Phi(\mathbf{X}_i^{\mathrm{tr}}) \right\rangle_{\mathcal{H}} \right| \nonumber \\
&\le \|m\|_{\mathcal{H}} \left\| \frac{1}{n_{\mathrm{tr}}} \sum_{i=1}^{n_{\mathrm{tr}}} (\hat{w}_i - \omega^*_i) \Phi(\mathbf{X}_i^{\mathrm{tr}}) \right\|_{\mathcal{H}} \nonumber \\
&\le \|m\|_{\mathcal{H}} \left( \left\| \hat{\mu}_{\Phi}^{\mathrm{tr}}(\hat{\mathbf{w}}) - \hat{\mu}_{\Phi}^{\mathrm{te}} \right\|_{\mathcal{H}} + \left\| \hat{\mu}_{\Phi}^{\mathrm{tr}}(\bm{\omega}^*) - \hat{\mu}_{\Phi}^{\mathrm{te}} \right\|_{\mathcal{H}} \right) \nonumber \\
&\le 2 \|m\|_{\mathcal{H}} \left\| \frac{1}{n_{\mathrm{tr}}} \sum_{i=1}^{n_{\mathrm{tr}}} \omega^*_i \Phi(\mathbf{X}_i^{\mathrm{tr}}) - \frac{1}{n_{\mathrm{te}}} \sum_{j=1}^{n_{\mathrm{te}}} \Phi(\mathbf{X}_j^{\mathrm{te}}) \right\|_{\mathcal{H}}. \label{eq:term3_opt}
\end{align}
Centering both empirical embeddings around the population target mean embedding $\mu_\Phi = \mathbb{E}_{P_{\mathrm{te}}}[\Phi(\mathbf{X})] = \mathbb{E}_{P_{\mathrm{tr}}}[\omega^*(\mathbf{X})\Phi(\mathbf{X})]$, we decompose:
\begin{equation}
\label{eq:mmd_decomp}
\left\| \frac{1}{n_{\mathrm{tr}}} \sum_{i=1}^{n_{\mathrm{tr}}} \omega^*_i \Phi(\mathbf{X}_i^{\mathrm{tr}}) - \frac{1}{n_{\mathrm{te}}} \sum_{j=1}^{n_{\mathrm{te}}} \Phi(\mathbf{X}_j^{\mathrm{te}}) \right\|_{\mathcal{H}}
\le \left\| \frac{1}{n_{\mathrm{tr}}} \sum_{i=1}^{n_{\mathrm{tr}}} \mathbf{Z}_i^{\mathrm{tr}} \right\|_{\mathcal{H}} + \left\| \frac{1}{n_{\mathrm{te}}} \sum_{j=1}^{n_{\mathrm{te}}} \mathbf{Z}_j^{\mathrm{te}} \right\|_{\mathcal{H}},
\end{equation}
where $\mathbf{Z}_i^{\mathrm{tr}} = \omega^*(\mathbf{X}_i^{\mathrm{tr}})\Phi(\mathbf{X}_i^{\mathrm{tr}}) - \mu_\Phi$ and $\mathbf{Z}_j^{\mathrm{te}} = \Phi(\mathbf{X}_j^{\mathrm{te}}) - \mu_\Phi$.

As established in the main text, $\|\Phi(\mathbf{x})\|_{\mathcal{H}} \le C$. Hence $\|\mu_\Phi\|_{\mathcal{H}} \le C$, giving the uniform bounds:
\[
\|\mathbf{Z}_i^{\mathrm{tr}}\|_{\mathcal{H}} \le \|\omega^*(\mathbf{X}_i^{\mathrm{tr}})\Phi(\mathbf{X}_i^{\mathrm{tr}})\|_{\mathcal{H}} + \|\mu_\Phi\|_{\mathcal{H}} \le BC + C \le 2BC,
\]
where the final inequality holds since $B = \sup_{\mathbf{x}}\omega^*(\mathbf{x}) \ge 1$ for any valid density ratio. For the target sequence, $\|\mathbf{Z}_j^{\mathrm{te}}\|_{\mathcal{H}} \le 2C$. Applying Lemma~\ref{lem:hilbert-mixing-blocking} separately to the source sequence with $R = 2BC$ at confidence level $\delta/6$, and to the target sequence with $R = 2C$ at confidence level $\delta/6$, each bound holds with probability at least $1 - \delta/6$. By a union bound, with probability at least $1 - \delta/3$:
\begin{equation}
\left\| \frac{1}{n_{\mathrm{tr}}} \sum_{i=1}^{n_{\mathrm{tr}}} \omega^*_i \Phi(\mathbf{X}_i^{\mathrm{tr}}) - \frac{1}{n_{\mathrm{te}}} \sum_{j=1}^{n_{\mathrm{te}}} \Phi(\mathbf{X}_j^{\mathrm{te}}) \right\|_{\mathcal{H}} \le C \Delta_{\mathrm{MMD}}(\delta),
\end{equation}
where the empirical RKHS feature-matching discrepancy $\Delta_{\mathrm{MMD}}(\delta)$ is given in closed form by:
\begin{align}
\label{eq:delta-mmd}
\Delta_{\mathrm{MMD}}(\delta) &:= K_H \sqrt{\log\frac{48}{\delta}} \Biggl[ B \sqrt{\frac{\log n_{\mathrm{tr}} + \log(6/\delta)}{n_{\mathrm{tr}}}} + \sqrt{\frac{\log n_{\mathrm{te}} + \log(6/\delta)}{n_{\mathrm{te}}}} \Biggr],
\end{align}
with the blocking constant $K_H > 0$ from Lemma~\ref{lem:hilbert-mixing-blocking}. 

For interpretability and physical intuition in the main text (Theorem~\ref{thm:kmm-dependent}), we abstract this exact non-asymptotic bound by defining the \emph{RKHS effective sample sizes} as:
\begin{equation}
\label{eq:rkhs_neff_def}
n_{\mathrm{eff, RKHS}}^{\mathrm{tr}} := \frac{n_{\mathrm{tr}}}{\log n_{\mathrm{tr}} + \log(6/\delta)}, \quad \text{and} \quad n_{\mathrm{eff, RKHS}}^{\mathrm{te}} := \frac{n_{\mathrm{te}}}{\log n_{\mathrm{te}} + \log(6/\delta)}.
\end{equation}
This definition allows us to express the MMD error conceptually as $\Delta_{\mathrm{MMD}}(\delta) \propto \frac{B}{\sqrt{n_{\mathrm{eff, RKHS}}^{\mathrm{tr}}}} + \frac{1}{\sqrt{n_{\mathrm{eff, RKHS}}^{\mathrm{te}}}}$. This formulation explicitly exposes how data dependence deflates the effective sample size for infinite-dimensional embeddings via the geometric blocking penalty, and reintroduces the target sequence sample size ($n_{\mathrm{te}}$) into the final theoretical guarantee.

Substituting \eqref{eq:delta-mmd} into \eqref{eq:term3_opt} yields:
\begin{equation}
\label{eq:bound_term2_final}
\left| \frac{1}{n_{\mathrm{tr}}} \sum_{i=1}^{n_{\mathrm{tr}}} (\hat{w}_i - \omega^*_i) m(\mathbf{X}_i^{\mathrm{tr}}) \right| \le 2C \|m\|_{\mathcal{H}} \Delta_{\mathrm{MMD}}(\delta).
\end{equation}

\paragraph{Final Synthesis (Proof of Theorem~\ref{thm:kmm-dependent}):}
Finally, we assemble the complete bound. We apply a union bound across the three failure events associated with Term 1 \eqref{eq:bound_term1}, Term 2 \eqref{eq:bound_term3}, and Term 3 \eqref{eq:bound_term2_final}. Since each event fails with probability at most $\delta/3$, their intersection holds with probability at least $1 - \delta$. 

By applying the triangle inequality to the initial decomposition, the absolute estimation error is strictly bounded by the sum of these three non-asymptotic components:
\begin{align}
\label{eq:appendix-final-bound}
\left| \hat{\mu}_{\mathrm{KMM}} - \mu_{\mathrm{te}} \right| 
&\le \underbrace{ B \sqrt{\frac{\log(6/\delta)}{2 n_{\mathrm{tr}}}} }_{\text{Term 1: } \mathcal{E}_{\mathrm{noise}}} \quad + \underbrace{ B \sqrt{\frac{\log(6/\delta)}{2 n_{\mathrm{eff}}^{\mathrm{tr}}}} }_{\text{Term 2: } \mathcal{E}_{\mathrm{shift}}} \nonumber \\
&\quad + \underbrace{ 2C K_H \sqrt{\log\frac{48}{\delta}} \|m\|_{\mathcal{H}} \left( \frac{B}{\sqrt{n_{\mathrm{eff, RKHS}}^{\mathrm{tr}}}} + \frac{1}{\sqrt{n_{\mathrm{eff, RKHS}}^{\mathrm{te}}}} \right) }_{\text{Term 3: Exact Non-Asymptotic } \mathcal{E}_{\mathrm{MMD}}}.
\end{align}
By abstracting the absolute universal constants ($2C$ and $K_H$) and the logarithmic confidence factor ($\sqrt{\log(48/\delta)}$) into the asymptotic $\mathcal{O}(\cdot)$ notation, we exactly recover the generalized statement of Theorem~\ref{thm:kmm-dependent} presented in the main text. This completes the proof.
\end{proof}

\section{Gaussian Process Effective Sample Size Derivations}
\label{app:gp_derivations}

This appendix provides the rigorous mathematical derivations for the Gaussian Process (GP) effective sample size introduced in Section~\ref{sec:density-smoothness}. We first establish the general marginalised covariance for a GP evaluated over correlated inputs. To mathematically formalise the fundamental smoothness trade-off, we then derive exact closed-form expressions and asymptotic limits for two contrasting covariance families: the infinitely smooth Squared-Exponential (SE) kernel and the rough Matérn-1/2 (Exponential) kernel. Finally, we contrast these stochastic realisations with the asymptotic behaviour of a general fixed deterministic function.

\subsection{Marginalized Covariance for GP Realizations (Squared-Exponential Case)}
\label{app:gp_se_derivation}

Suppose $u$ is a realization of a zero-mean GP with a squared-exponential covariance $k_h(\mathbf{x}, \mathbf{x}') = \sigma_u^2 \exp(-\|\mathbf{x}-\mathbf{x}'\|^2 / 2h^2)$. To maintain theoretical consistency with the compact domain $\mathcal{X}$ (Assumption~\ref{ass:covariate_shift}), we formally assume the Gaussian inputs are truncated to a sufficiently large compact bounding box. Following standard truncation arguments in GP learning theory \citep{srinivas2009gaussian}, the exponentially small probability mass lost in the tails ensures our analytical moments hold up to a negligible approximation error. 

If the covariates are drawn from a stationary GP $\mathbf{X}_s \sim \mathcal{N}(\bm{\mu}, \sigma_x^2 \mathbf{I}_d)$, the difference between two inputs is strictly Gaussian:
\begin{equation}
\mathbf{X}_s - \mathbf{X}_t \sim \mathcal{N}\left(\mathbf{0}, 2\sigma_x^2\{1 - r(\|s-t\|)\}\mathbf{I}_d\right).
\end{equation}
Integrating over the distribution of the $d$-dimensional inputs, the prior-averaged kernel covariance becomes:
\begin{equation}
\mathbb{E}_{\mathbf{X}} [k_h(\mathbf{X}_s, \mathbf{X}_t)] = \sigma_u^2 \det\left( \mathbf{I}_d + \frac{2\sigma_x^2\{1 - r(\|s-t\|)\}}{h^2}\mathbf{I}_d \right)^{-1/2}.
\end{equation}
Subtracting the independent-input baseline gives the exact covariance of the centred process:
\begin{equation}
\label{eq:app-gp-centered-covariance}
\bar{\gamma}_u(s,t) = \sigma_u^2 \left[ \det\left( \mathbf{I}_d + \frac{2\sigma_x^2\{1 - r(\|s-t\|)\}}{h^2}\mathbf{I}_d \right)^{-1/2} - \det\left( \mathbf{I}_d + \frac{2\sigma_x^2}{h^2}\mathbf{I}_d \right)^{-1/2} \right].
\end{equation}

Letting $a_h = 2\sigma_x^2/h^2$ and $r_k = r(k) = \exp(-k^2/2\ell_x^2)$, the determinant simplifies to $(1 + a_h\{1-r_k\})^d$. Consequently, the autocorrelation function in $\mathbb{R}^d$ takes the exact closed form:
\begin{equation}
\label{eq:gp-centered-rho-vector}
\bar{\rho}_u(k) = \frac{ \left(1 + a_h\{1 - r_k\}\right)^{-d/2} - \left(1 + a_h\right)^{-d/2} }{ 1 - \left(1 + a_h\right)^{-d/2} }.
\end{equation}
Defining $\kappa_{h,d} = (1+a_h)^{-d/2}$, the infinite autocorrelation sum becomes:
\begin{equation}
\bar{S}(a_h, \ell_x, d) = \sum_{k=1}^{\infty} \bar{\rho}_u(k) = \sum_{k=1}^{\infty} \frac{ (1+a_h(1-r_k))^{-d/2} - \kappa_{h,d} }{ 1 - \kappa_{h,d} }.
\end{equation}

\textbf{Short-Range Dependence Limit ($\ell_x \ll 1$):} The first lag dominates the sum. Using a first-order expansion, we find:
\begin{equation}
\bar{S}(a_h, \ell_x) = L(a_h) \exp\left(-\frac{1}{2\ell_x^2}\right) + \mathcal{O}\left(\exp\left(-\frac{1}{\ell_x^2}\right)\right), \quad \text{where } L(a_h) = \frac{\sqrt{1+a_h}+1}{2(1+a_h)}.
\end{equation}

\textbf{Long-Range Dependence Limit ($\ell_x \gg 1$):} The sum is approximated by the integral via the Euler-Maclaurin formula:
\begin{equation}
\bar{S}(a_h, \ell_x) = \ell_x \mathcal{C}_d(a_h) - \frac{1}{2} + \mathcal{O}(\ell_x^{-1}),
\end{equation}
where the constant $\mathcal{C}_d(a_h)$ explicitly captures the GP smoothness effect:
\begin{equation}
\mathcal{C}_d(a_h) = \int_0^\infty \frac{ \left(1 + a_h\{1 - \exp(-u^2/2)\}\right)^{-d/2} - \left(1 + a_h\right)^{-d/2} }{ 1 - \left(1 + a_h\right)^{-d/2} } \,du.
\end{equation}

Analyzing $\mathcal{C}_d(a_h)$ yields two sub-regimes:
1. If $a_h \ll 1$ (Smooth GP), Taylor expansion yields $\mathcal{C}_d(a_h) = \sqrt{\frac{\pi}{2}} - \frac{3a_h}{4}\left(\sqrt{\frac{\pi}{2}} - \frac{\sqrt{\pi}}{2}\right) + \mathcal{O}(a_h^2)$.
2. If $a_h \gg 1$ (Rough GP), asymptotic analysis yields $\mathcal{C}_d(a_h) = \frac{\log a_h}{\sqrt{2a_h}} + \mathcal{O}(a_h^{-1/2})$.

\subsection{Effective Sample Size for the Matérn Kernel Family (e.g., Exponential Kernel)}
\label{app:matern_derivation}

To demonstrate that our effective sample size framework seamlessly extends beyond the infinitely smooth Squared-Exponential (SE) kernel, we analyse the Matérn family of covariance functions. The Matérn class is governed by a smoothness parameter $\nu > 0$. As $\nu \to \infty$, it recovers the SE kernel. However, for finite $\nu$, the functions are rougher (only $\lfloor \nu - 0.5 \rfloor$ times mean-square differentiable).

A widely used instantiation is the Matérn-$1/2$ kernel, which corresponds to the Exponential (or Laplace) covariance function: $k_{1/2}(\mathbf{x}, \mathbf{x}') = \sigma_u^2 \exp\left( -\|\mathbf{x} - \mathbf{x}'\| / h \right)$. This models rough, non-differentiable processes.

For scalar inputs ($d=1$), we can derive an exact analytical closed form for the marginalised covariance, similar to the SE case. As established in \eqref{eq:app-gp-centered-covariance}, the difference between two inputs is $\Delta X = X_s - X_t \sim \mathcal{N}(0, s_k^2)$, where the variance is $s_k^2 = 2\sigma_x^2\{1 - r(k)\}$. The prior-averaged kernel covariance requires evaluating the Gaussian integral over the absolute value:
\begin{equation}
\mathbb{E}_{X} [k_{1/2}(X_s, X_t)] = \int_{-\infty}^{\infty} \sigma_u^2 \exp\left(-\frac{|x|}{h}\right) \frac{1}{\sqrt{2\pi s_k^2}} \exp\left(-\frac{x^2}{2s_k^2}\right) dx.
\end{equation}
By completing the square in the exponent and exploiting the symmetry of the Gaussian distribution, this integral evaluates exactly to:
\begin{equation}
\mathbb{E}_{X} [k_{1/2}(X_s, X_t)] = \sigma_u^2 \exp\left(\frac{s_k^2}{2h^2}\right) \operatorname{erfc}\left( \frac{s_k}{\sqrt{2}h} \right),
\end{equation}
where $\operatorname{erfc}(z) = \frac{2}{\sqrt{\pi}} \int_z^\infty e^{-t^2} dt$ is the complementary error function.

To express the induced autocorrelation $\bar{\rho}_u(k)$, let $\gamma_k = \frac{s_k}{\sqrt{2}h} = \frac{\sigma_x}{h}\sqrt{1 - r_k}$, and let $\gamma_\infty = \frac{\sigma_x}{h}$ denote the independent-input baseline limit. The exact centered autocorrelation for the Exponential GP is:
\begin{equation}
\bar{\rho}_{u,\text{Exp}}(k) =  \frac{ \exp(\gamma_k^2)\operatorname{erfc}(\gamma_k) - \exp(\gamma_\infty^2)\operatorname{erfc}(\gamma_\infty) }{ 1 - \exp(\gamma_\infty^2)\operatorname{erfc}(\gamma_\infty) }.
\label{eq:gp-exponential-rho}
\end{equation}

\paragraph{Asymptotic Smoothness Trade-off.} 
This closed-form derivation mathematically formalizes the smoothness trade-off. For large input deviations relative to the length-scale ($\gamma_k \gg 1$), the complementary error function admits the strict asymptotic expansion:
\begin{equation}
\exp(\gamma_k^2)\operatorname{erfc}(\gamma_k) = \frac{1}{\gamma_k \sqrt{\pi}} \left( 1 - \mathcal{O}(\gamma_k^{-2}) \right).
\end{equation}
Although single-lag autocorrelations for both kernels asymptotically scale as $\mathcal{O}(1/\gamma_k)$ for large input deviations (since $a_h\{1-r_k\} = 2\gamma_k^2$, yielding $(1+2\gamma_k^2)^{-1/2} = \frac{1}{\sqrt{2}\gamma_k}(1 - \mathcal{O}(\gamma_k^{-2}))$ for the SE model), their cumulative persistence sums behave fundamentally differently across lags. Because the SE kernel models an infinitely differentiable process, its spectral density decays super-exponentially; in the continuum limit, this parabolic smoothness near zero lag produces an extended cumulative dependence with an additional logarithmic enhancement $\mathcal{C}_d(a_h) = \frac{\log a_h}{\sqrt{2a_h}} + \mathcal{O}(a_h^{-1/2})$. Conversely, the non-differentiable Matérn-1/2 process introduces high-frequency roughness that enforces faster sequence decorrelation across neighboring points, eliminating this logarithmic inflation and resulting in a smaller infinite sum $\bar{S}$ and a larger, more favorable effective sample size.

\subsection{Asymptotic Analysis of the GP Effective Sample Size for Fixed Functions}
\label{app:general_fixed_function}

To contextualise the GP realisation results, consider a general fixed deterministic function evaluated over the same spatio-temporal data. Suppose the covariates are generated by a stationary Gaussian process $\mathbf{X}_s \sim \mathcal{N}(\bm{\mu}, \sigma_x^2 \mathbf{I}_d)$ with spatial/temporal covariance:
\begin{equation*}
\operatorname{Cov}(\mathbf{X}_s, \mathbf{X}_t) = \sigma_x^2 \mathbf{I}_d \exp\left(-\frac{\|s-t\|^2}{2\ell_x^2}\right),
\end{equation*}
where $\ell_x$ determines the persistence length. Assuming $f$ is continuously differentiable with bounded second derivatives, a first-order Taylor expansion around the mean $\bm{\mu}$ yields the leading term:
\begin{equation*}
\operatorname{Cov}(f(\mathbf{X}_s), f(\mathbf{X}_t)) = \sigma_x^2 \|\nabla f(\bm{\mu})\|^2 r(\|s-t\|) + \mathcal{O}(\sigma_x^4).
\end{equation*}
Consequently, the leading term of the function's autocorrelation strictly matches the input correlation up to higher-order residuals: $\rho_f(k) = r(k) + \mathcal{O}(\sigma_x^2)$.

For a squared-exponential input autocorrelation, the infinite autocorrelation sum is explicitly given by the Jacobi theta function $\vartheta_3$:
\begin{align*}
S(\ell_x) = \sum_{k=1}^{\infty} \exp\left(-\frac{k^2}{2\ell_x^2}\right) = \frac{1}{2} \left[ \vartheta_3\left(0, \exp\left\{-\frac{1}{2\ell_x^2}\right\}\right) - 1 \right].
\end{align*}
Therefore, the exact asymptotic effective sample size is expressed as:
\begin{align*}
n_{\mathrm{eff}}(\ell_x) = \frac{n}{1 + 2S(\ell_x)} = \frac{n}{\vartheta_3\left(0, \exp\left\{-1/(2\ell_x^2)\right\}\right)}.
\end{align*}
In the short-range limit ($\ell_x \to 0$), the first lag dominates, yielding $n_{\mathrm{eff}} = n(1 - \mathcal{O}(\exp(-1/2\ell_x^2)))$. Conversely, under long-range dependence ($\ell_x \to \infty$), approximating the sum with an integral via the Euler-Maclaurin formula gives $1 + 2S(\ell_x) = \sqrt{2\pi}\ell_x + \mathcal{O}(1)$. Thus, the effective sample size scales asymptotically as $n_{\mathrm{eff}}(\ell_x) = \frac{n}{\sqrt{2\pi}\ell_x + \mathcal{O}(1)}$.

\section{Extension to Hypothesis Classes via Empirical Processes}
\label{app:rademacher}

While Theorem~\ref{thm:kmm-dependent} provides a pointwise guarantee for a fixed regression function $m$ or a fixed loss $\ell$, evaluating a data-dependent learned predictor requires a simultaneous uniform guarantee over a hypothesis class $\mathcal{F}$. For independent data, this is governed by standard Rademacher complexity. For dependent $\beta$-mixing sequences, applying a naive substitution $n \rightsquigarrow n_{\mathrm{eff}}$ is theoretically insufficient. Instead, uniform convergence must be established using \emph{Block Rademacher Complexity} \citep{mohri2009rademacher, kuznetsov2015learning}.

Following the blocking construction from Lemma~\ref{lem:hilbert-mixing-blocking}, we partition the sequence into alternating blocks of size $q$, with $L = \lfloor n/(2q) \rfloor$. By Berbee's coupling lemma, we decouple the dependent blocks into an independent block sequence $B_1^\ast, \dots, B_L^\ast$, where each block contains $q$ observations: $B_j^\ast = (\mathbf{X}_{j,1}^\ast, \dots, \mathbf{X}_{j,q}^\ast)$. The empirical Block Rademacher Complexity of $\mathcal{F}$ is defined as:
\begin{equation}
\label{eq:block-rademacher}
\mathfrak{R}_{L}(\mathcal{F}, q) = \frac{1}{L} \mathbb{E}_{\bm{\sigma}} \left[ \sup_{f \in \mathcal{F}} \sum_{j=1}^{L} \sigma_j \left( \frac{1}{q} \sum_{k=1}^{q} f(\mathbf{X}_{j,k}^\ast) \right) \right],
\end{equation}
where $\sigma_1, \dots, \sigma_L$ are independent Rademacher variables.

Applying McDiarmid's inequality to the coupled independent blocks and accounting for odd and even sub-sequences via the triangle inequality, the uniform deviation over $\mathcal{F}$ is bounded with probability at least $1-\delta$ by:
\begin{equation}
\label{eq:uniform-mixing-bound}
\sup_{f\in\mathcal{F}} \left| \frac{1}{n}\sum_{i=1}^n f(\mathbf{X}_i) - \mathbb{E}[f(\mathbf{X})] \right| \le 2 \mathfrak{R}_{L}(\mathcal{F}, q) + 3 A \sqrt{\frac{\log(4/\delta)}{2 L}} + 2 L \beta(q),
\end{equation}
where $A = \sup_{f \in \mathcal{F}} (\sup_{\mathbf{x}} f(\mathbf{x}) - \inf_{\mathbf{x}} f(\mathbf{x}))$ denotes the uniform bound on the function range. By choosing the block size to satisfy $q = \lceil \frac{\log n + \log(4 c_1 / \delta)}{c_2} \rceil$, the mixing penalty becomes negligible: $2 L \beta(q) \le \delta/2$. The number of effective independent blocks $L = \lfloor n/(2q) \rfloor = \mathcal{O}(n / \log n)$ then governs the sample-complexity geometry of the uniform bound.

\paragraph{Application to RKHS and GP Function Classes.}
For an RKHS ball $\mathcal{F}_R = \{f \in \mathcal{H} : \|f\|_{\mathcal{H}} \le R\}$ with a kernel bounded by $C^2$, let $\bar{\Phi}_j = \frac{1}{q} \sum_{k=1}^q \Phi(\mathbf{X}_{j,k}^\ast)$ be the block feature mean of the $j$-th block. Since $\|\Phi(\mathbf{x})\|_{\mathcal{H}} \le C$, it follows that $\|\bar{\Phi}_j\|_{\mathcal{H}} \le C$. By Cauchy-Schwarz, the block Rademacher complexity evaluates cleanly to:
\begin{equation}
\mathfrak{R}_{L}(\mathcal{F}_R, q) = \frac{1}{L} \mathbb{E}_{\bm{\sigma}} \left[ \sup_{\|f\|_{\mathcal{H}} \le R} \left\langle f, \sum_{j=1}^L \sigma_j \bar{\Phi}_j \right\rangle_{\mathcal{H}} \right] \le \frac{R}{L} \mathbb{E}_{\bm{\sigma}} \left\| \sum_{j=1}^L \sigma_j \bar{\Phi}_j \right\|_{\mathcal{H}} \le \frac{RC}{\sqrt{L}}.
\end{equation}
Notice that replacing the nominal sample size $n$ with the effective number of blocks $L = \mathcal{O}(n / \log n)$ inflates the statistical complexity, mirroring the effective sample size deflation observed in our pointwise bounds. 

This formulation mathematically connects KMM risk estimation with empirical process theory under dependent sampling. While the block size $q$ is determined entirely by the physical mixing rate of the covariate process $\{\mathbf{X}_t\}$, functional smoothness directly controls the complexity radius $R = \|f\|_{\mathcal{H}}$: smoother function classes enforce smaller RKHS norms, thereby substantially reducing the complexity numerator $\mathfrak{R}_L$. However, as established in Section~\ref{sec:density-smoothness}, smoother functions simultaneously retain more intra-sequence autocorrelation, deflating the pointwise variance-based $n_{\mathrm{eff}}$. Across both pointwise and uniform bounds, the density-ratio overlap penalty $B$ remains an irreducible multiplicative scale, confirming that neither hypothesis smoothness nor complex empirical-process machinery can bypass severe support mismatch.

\section{Artificial Data Experimental Results and Implementation Details}
\label{app:additional_experiments}

This section provides a comprehensive breakdown of the synthetic experimental setups, expanded empirical validations, and rigorous robustness stress tests that were omitted from the main text due to space constraints. We first detail the exact data-generating processes, followed by the validation of the Target $Y$ effective sample size. We then present full-page heatmaps of our theoretical calibration and conclude with extensive ablations over hyperparameters, non-differentiable functional families, and non-stationary correlation structures.

\subsection{Detailed Synthetic Experimental Setup}
\label{app:exp_setup}

To thoroughly evaluate the statistical properties of our dependent KMM framework, we construct a highly controlled synthetic environment where ground-truth expectations, density ratios, and effective sample sizes are analytically tractable. 

\paragraph{Data-Generating Process for Covariates}
We generate continuous covariate sequences $X_t \in \mathbb{R}$ for our 1D benchmark sweeps and hyperparameter ablations (extended to $\mathbf{X}_t \in \mathbb{R}^2$ for the 2D spatial calibration heatmaps in Appendix~\ref{app:heatmaps}). To simulate temporal or spatial persistence, source covariates are sampled along a stationary Gaussian trajectory:
\begin{equation}
X_t^{\mathrm{tr}} = \mu_{\mathrm{tr}} + \tau_{\mathrm{tr}}\xi_t, \quad \text{where} \quad \bm{\xi} \sim \mathcal{N}(\mathbf{0}, \mathbf{\Sigma}).
\end{equation}
The covariance matrix $\mathbf{\Sigma} \in \mathbb{R}^{n_{\mathrm{tr}} \times n_{\mathrm{tr}}}$ enforces a Squared-Exponential (RBF) autocorrelation structure across sequence indices:
\begin{equation}
\operatorname{Cov}(\xi_s, \xi_t) = \sigma_0^2 \exp\left(-\frac{|s-t|^2}{2\ell_x^2}\right).
\end{equation}
Here, $\ell_x$ dictates the input correlation length and $\sigma_0^2 = 1.0$ is the baseline variance. Target covariates $X_t^{\mathrm{te}}$ are sampled under covariate shift with marginal mean $\mu_{\mathrm{te}}$. In our baseline benchmark sweeps, we set default sample sizes $n_{\mathrm{tr}} = n_{\mathrm{te}} = 96$, shift parameters $\mu_{\mathrm{tr}} = 0.0$, $\mu_{\mathrm{te}} = 0.85$, scaling factors $\tau_{\mathrm{tr}} = \tau_{\mathrm{te}} = 1.0$, and default sequence correlation length $\ell_x = 3.0$.

\paragraph{Response Surfaces and Estimands}
To evaluate the GP-induced smoothness correction, the true response is tested under two distinct structural environments:
\begin{enumerate}
    \item \textbf{Matched GP Realization}: The raw response $f_{\mathrm{raw}}$ is synthesized via Random Fourier Features (using 512 Fourier basis functions) matching a stationary Gaussian Process with length-scale $h$ ($h = 1.5$ by default). To enforce realistic bounded support, the function is passed through a non-linear activation:
    \begin{equation}
    f_0(x) = 0.5 + 0.4\tanh(0.9x + f_{\mathrm{raw}}(x)).
    \end{equation}
    Observed targets include additive Gaussian noise $Y = \operatorname{clip}(f_0(X) + \varepsilon, 0, 1)$ with $\varepsilon \sim \mathcal{N}(0, \sigma_\varepsilon^2)$ and noise level $\sigma_\varepsilon = 0.04$.
    \item \textbf{Deterministic Misspecified Surface}: To verify robustness outside GP priors, we also evaluate a fixed deterministic harmonic function $f_0(\mathbf{x}) = \sin(\omega x_1) + 0.5\cos(2\omega x_2)$ with spatial frequency $\omega \in [0.5, 3.0]$.
\end{enumerate}
For target mean estimation, we evaluate the raw response $Y$. For target risk estimation, we utilize a fixed, imperfect predictor $g(x) = 0.65 f_0(x) + 0.15$ and evaluate the squared loss $\ell(g(x), y) = (g(x) - y)^2$. 

Ground-truth target expectations $\theta_Y = \mathbb{E}_{P_{\mathrm{te}}}[Y]$ and $\theta_{\mathrm{Risk}} = \mathbb{E}_{P_{\mathrm{te}}}[\ell(g(X), Y)]$ are computed via high-precision Monte Carlo integration over $150{,}000$ independent target samples drawn from $P_{\mathrm{te}}$.

\paragraph{KMM Optimization and Variogram Calibration}
Kernel Mean Matching weights $\mathbf{w} \in \mathbb{R}^{n_{\mathrm{tr}}}$ are estimated by solving the quadratic program:
\begin{equation}
\min_{\mathbf{w}} \frac{1}{2}\mathbf{w}^\top \left(\mathbf{K}_{ss} + \lambda \mathbf{I}\right) \mathbf{w} - \bm{\kappa}^\top \mathbf{w} \quad \text{s.t.} \quad 0 \le w_i \le B, \quad \sum_{i=1}^{n_{\mathrm{tr}}} w_i = n_{\mathrm{tr}},
\end{equation}
where $\kappa_i = \frac{n_{\mathrm{tr}}}{n_{\mathrm{te}}} \sum_{j=1}^{n_{\mathrm{te}}} k(X_i^{\mathrm{tr}}, X_j^{\mathrm{te}})$. We utilize an RBF kernel $k(x, x') = \exp(-(x-x')^2 / (2\sigma^2))$ with bandwidth $\sigma = 1.0$, ridge regularization $\lambda = 10^{-4}$, and default clipping bound $B = 15.0$. The optimization problem is solved using the interior-point solver in \texttt{cvxopt}.

To isolate confidence radius calibration from empirical variogram estimation noise during finite sweeps, the functional length-scales $h_Y$ and $h_\ell$ are pre-calibrated via empirical variogram regression over an independent calibration pool of 2,500 source observations. All Monte Carlo sweeps are conducted over $N = 300$ independent trials initialized with random seed 42.

\paragraph{Baseline Formulations and Implementation}
For robust benchmarking, we compare our proposed GP-calibrated KMM radius against multiple standard estimators:
\begin{itemize}
    \item \textbf{Unweighted}: Direct sample average over source data $\frac{1}{n_{\mathrm{tr}}} \sum_{i=1}^{n_{\mathrm{tr}}} Y_i^{\mathrm{tr}}$, which ignores covariate shift.
    \item \textbf{Nominal $n$}: The classical i.i.d. KMM confidence interval, which weights samples via KMM but uses the nominal sample size $n_{\mathrm{tr}}$ in Wald's variance estimator.
    \item \textbf{Input $n_{\mathrm{eff}}$}: A dependence correction that penalizes variance using the raw input autocorrelation $r_k$: $n_{\mathrm{eff}}(\ell_x) = n_{\mathrm{tr}} / [1 + 2\sum_{k=1}^{n_{\mathrm{tr}}-1}(1 - k/n_{\mathrm{tr}})r_k]$, ignoring functional smoothness.
    \item \textbf{Block-Bootstrap}: A moving-block bootstrap with block length $q = \lceil n_{\mathrm{tr}}^{1/3} \rceil$ ($q=6$ at default $n=96$) using 200 bootstrap resamples.
    \item \textbf{Strict Bound (Theorem~\ref{thm:kmm-dependent})}: The non-asymptotic concentration bound evaluating the conservative radius $B \sqrt{\log(6/\alpha) / (2 n_{\mathrm{eff}}^{\mathrm{GP}})}$.
    \item \textbf{Oracle IW}: Importance weighting using the true analytical density ratio $w^*(x) = p_{\mathrm{te}}(x) / p_{\mathrm{tr}}(x)$.
    \item \textbf{Oracle i.i.d.}: Target mean and risk evaluated directly on independent target samples drawn from $P_{\mathrm{te}}$.
\end{itemize}

\subsection{Extended Evaluation Metrics and Full Benchmark Results}
\label{app:full_metrics}

In Section~\ref{sec:experiments}, we focused on Coverage and Radius to evaluate the confidence intervals, as point estimates are identical across all KMM-based variance estimators. Table~\ref{tab:main-benchmark-full} provides the complete set of evaluation metrics, including Mean Absolute Error (MAE) for point estimation and the Winkler Score (WS) for comprehensive interval evaluation.

The Winkler Score \citep{winkler1972decision} is a proper scoring rule for prediction intervals that rewards narrow intervals but heavily penalizes intervals that fail to cover the true value. For a nominal $1-\alpha$ confidence interval $[L, U]$ (where $\alpha=0.05$) and a true target value $y$, it is defined as:
\begin{equation}
\text{WS} = (U - L) + \frac{2}{\alpha}(L - y)\mathbb{I}(y < L) + \frac{2}{\alpha}(y - U)\mathbb{I}(y > U).
\end{equation}
Lower WS indicates better interval calibration. As shown in Table~\ref{tab:main-benchmark-full}, while the \textit{Strict Bound (Theorem~\ref{thm:kmm-dependent})} guarantees full coverage, its massive worst-case width results in a severely degraded WS ($>10.0$). Similarly, while \textit{Input $n_{\mathrm{eff}}$} is valid, it produces unnecessarily wide intervals (WS: $0.339$ for $Y$; $0.031$ for Risk). Our proposed GP $n_{\mathrm{eff}}$ provides the optimal balance, achieving the lowest WS ($0.276$ for $Y$; $0.026$ for Risk) among valid estimators by analytically incorporating functional decorrelation.

\begin{table*}[ht]
\centering
\caption{\textbf{Full Benchmark Results including Point Estimation and Winkler Score.} Point estimates (MAE) are identical for Nominal $n$, Block-Bootstrap, Input $n_{\mathrm{eff}}$, GP $n_{\mathrm{eff}}$, and the Strict Bound, as they rely on the exact same KMM weights. An asterisk ($^*$) indicates empirical undercoverage significantly below the nominal 0.95 level ($< 0.925$). Winkler Score (WS, lower is better) penalizes both interval width and undercoverage. Bold highlights the best valid practical estimator (tightest radius and lowest WS while maintaining valid coverage). Strict Thm 1 and Oracle methods are excluded from best-method highlighting.}
\label{tab:main-benchmark-full}
\vspace{1.5mm}
\resizebox{\textwidth}{!}{%
\begin{tabular}{@{}lcccccccc@{}}
\toprule
& \multicolumn{4}{c}{\textbf{Target Mean ($Y$)}} & \multicolumn{4}{c}{\textbf{Target Risk ($L$)}} \\
\cmidrule(lr){2-5} \cmidrule(lr){6-9}
\textbf{Method} & \textbf{MAE} & \textbf{Coverage} & \textbf{Radius} & \textbf{WS} $\downarrow$ & \textbf{MAE} & \textbf{Coverage} & \textbf{Radius} & \textbf{WS} $\downarrow$ \\
\midrule
Unweighted         & 0.1138 & 0.017$^*$ & 0.0362 & 3.182 & 0.0048 & 0.080$^*$ & 0.0019 & 0.119 \\
Nominal $n$        & 0.0196 & 0.917$^*$ & 0.0626 & 0.134 & 0.0025 & 0.853$^*$ & 0.0055 & 0.024 \\
Block-Bootstrap    & 0.0196 & 1.000     & 0.4538 & 0.908 & 0.0025 & 0.967     & 0.0133 & \textbf{0.026} \\
Input $n_{\mathrm{eff}}$ & 0.0196 & 1.000     & 0.1696 & 0.339 & 0.0025 & 0.987     & 0.0150 & 0.031 \\
\textbf{GP $n_{\mathrm{eff}}$ (Ours)} & 0.0196 & 1.000 & \textbf{0.1380} & \textbf{0.276} & 0.0025 & 0.970 & 0.0120 & \textbf{0.026} \\
Strict Thm 1       & 0.0196 & 1.000     & 5.2189 & 10.438& 0.0025 & 1.000     & 5.1220 & 10.244 \\
\midrule
Oracle IW          & 0.0334 & 0.477$^*$ & 0.0275 & 0.590 & 0.0027 & 0.627$^*$ & 0.0033 & 0.033 \\
Oracle i.i.d.      & 0.0115 & 0.903$^*$ & 0.0257 & 0.072 & 0.0010 & 0.917$^*$ & 0.0024 & 0.006 \\
\bottomrule
\end{tabular}}
\end{table*}

\subsection{Validation under Deterministic Misspecification}
\label{app:target_det_validation}

In the main text (Figure~\ref{fig:neff_validation}), we demonstrated that our closed-form GP effective sample size formula accurately matches empirical variance when the underlying response surface is drawn from a matched Gaussian Process. 

To verify that this theoretical result is robust and not merely an artifact of Gaussian assumptions, Figure~\ref{fig:neff_det} presents the validation on a deterministic harmonic surface $f_0(\mathbf{x}) = \sin(\omega x_1) + 0.5\cos(2\omega x_2)$. Despite the absence of a stochastic GP prior generating the labels, the empirical variance-based $n_{\mathrm{eff}}$ for both Target $Y$ and Risk evaluation remains tightly concentrated along the diagonal identity line. 
This indicates that the GP surrogate can remain informative beyond the matched GP setting, although its accuracy varies with the dependence and functional regimes considered below.

\begin{figure}[!htb]
\centering
\includegraphics[width=0.6\textwidth]{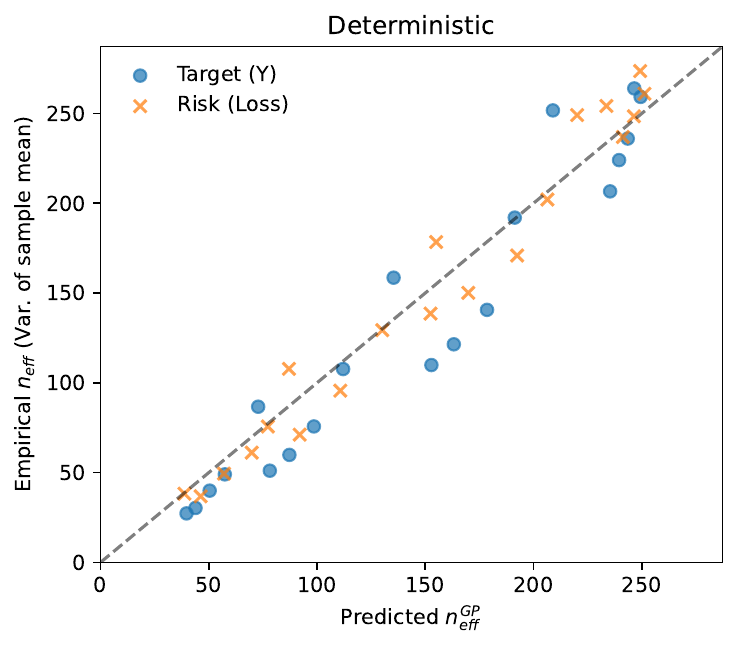}
\caption{\textbf{Validation under Deterministic Misspecification.} Empirical variance-based $n_{\mathrm{eff}}$ vs. our predicted $n_{\mathrm{eff}}^{\mathrm{GP}}$ for both Target $Y$ and Risk when evaluated on a deterministic, non-GP harmonic surface. The surrogate formula remains highly predictive.}
\label{fig:neff_det}
\end{figure}

\subsection{Effective Sample Size Heatmaps}
\label{app:heatmaps}

To ensure that the excellent calibration observed in our scatter plots is not an artifact of a specific, "lucky" hyperparameter combination, we evaluate the ratio of the empirical effective sample size to our predicted effective sample size: $n_{\mathrm{eff}}^{\mathrm{emp}} / n_{\mathrm{eff}}^{\mathrm{GP}}(h)$. A ratio of exactly $1.0$ indicates perfect theoretical calibration. 

We compute this ratio across a dense, two-dimensional grid varying both the input dependence $\ell_x \in [0.5, 4.0]$ and the functional complexity parameters ($h \in [0.5, 4.0]$ for GP realizations, and spatial frequency $\omega \in [0.5, 3.0]$ for deterministic surfaces). 


Figures~\ref{fig:heatmap_gp} and~\ref{fig:heatmap_det} display these heatmaps across the full parameter grids. The empirical to-predicted ratio is mostly within approximately $[0.7,1.25]$, although larger deviations occur in some strong-dependence and smooth-function regimes. Ratios below one indicate that the closed-form approximation overestimates the empirical effective sample size in these settings. Overall, the results show that the GP approximation tracks the empirical effective sample size across a broad range of regimes, while also identifying conditions under which its calibration deteriorates.

\begin{figure*}[!htb]
\centering
\includegraphics[width=0.9\textwidth]{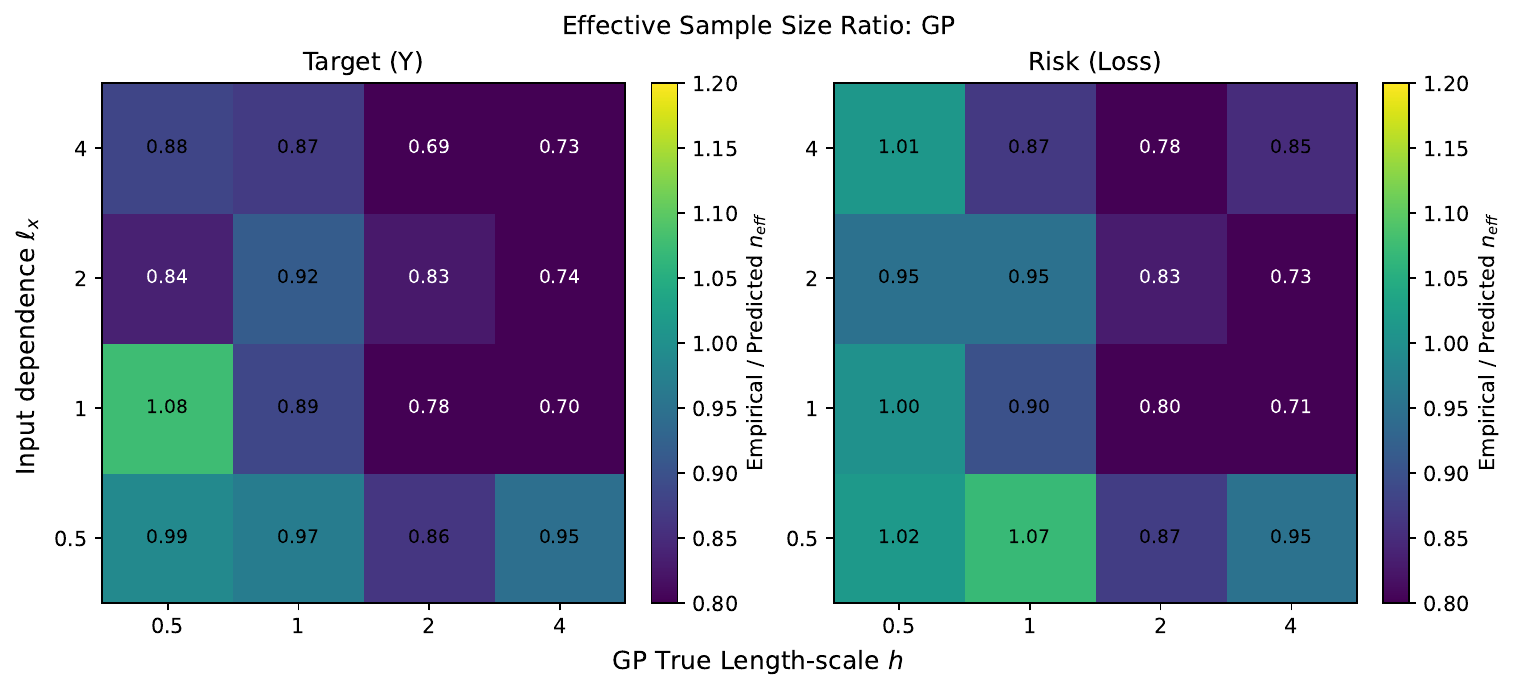}
\caption{\textbf{GP Realizations (Full Grid).}
Ratio of the empirical to predicted $n_{\mathrm{eff}}$ across diverse GP
length-scales ($h$) and input dependencies ($\ell_x$). Ratios near $1.0$
indicate close calibration. The approximation tracks the empirical effective sample size across much of the grid, with larger deviations in some strong-dependence and smooth-function regimes.}
\label{fig:heatmap_gp}
\end{figure*}

\begin{figure*}[!htb]
\centering
\includegraphics[width=0.9\textwidth]{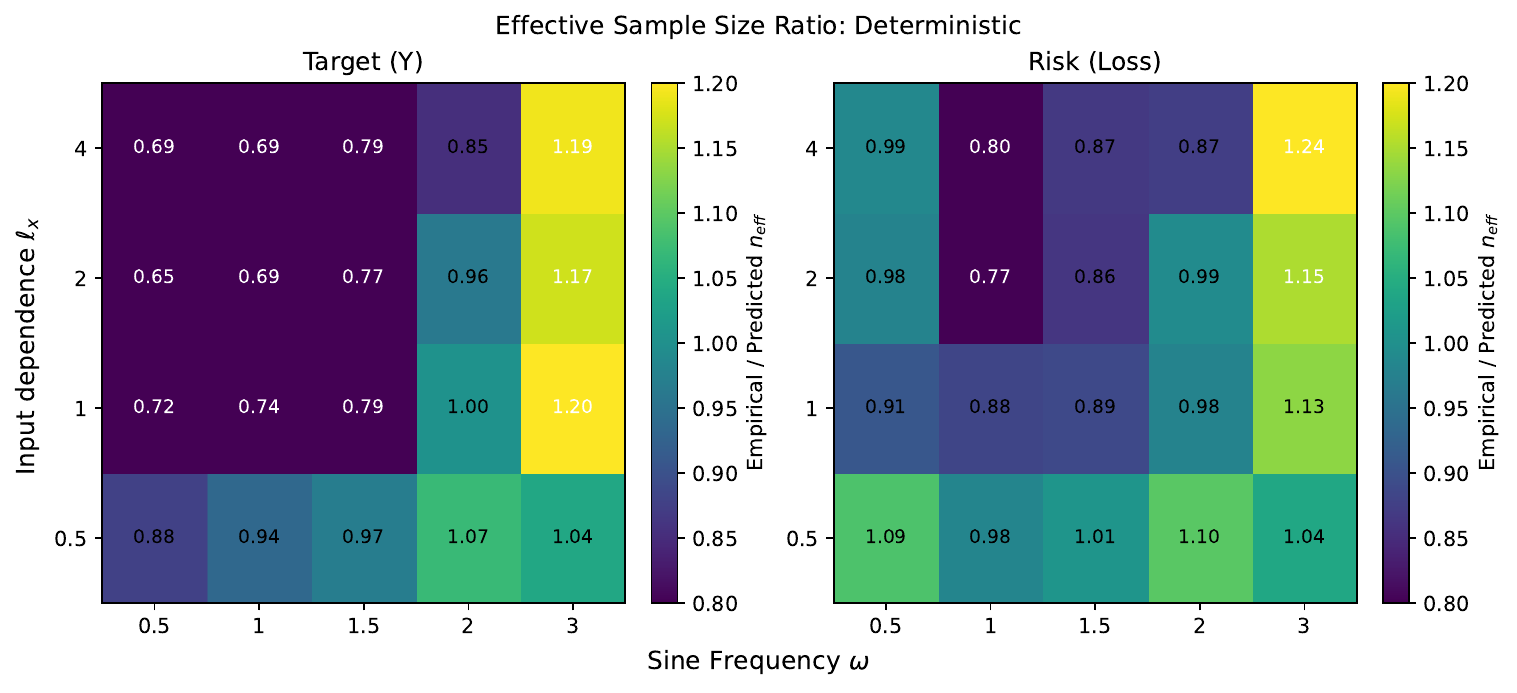}
\caption{\textbf{Deterministic Misspecification (Full Grid).} Ratio of the empirical to predicted $n_{\mathrm{eff}}$ when the true function is a deterministic harmonic surface with varying frequency $\omega$. The GP surrogate remains informative beyond the matched stochastic-GP setting, although its calibration varies across dependence and frequency regimes.}
\label{fig:heatmap_det}
\end{figure*}

\subsection{Hyperparameters Ablations}
\label{app:param_ablations}

To further dissect the behavior of the GP-induced effective sample size, we evaluate the KMM confidence intervals across comprehensive sweeps of sample size ($n$, Table~\ref{tab:n_scaling_risk}), input dependence severity ($\ell_x$), functional smoothness ($h$), and maximum clipping bounds ($B$). Across all ablations, we report interval coverage, interval radius, and the Winkler Score (WS).

We evaluate interval calibration using the Winkler Score (WS; detailed in Appendix~\ref{app:full_metrics}), which penalizes intervals that are either overly wide or fail to cover the target.

To evaluate validity, we track empirical coverage over $N = 300$ independent Monte Carlo repetitions. For nominal $0.95$ intervals, the binomial two-sigma lower bound is approximately $0.95 - 1.96 \sqrt{0.95 \times 0.05 / 300} \approx 0.925$. We flag any coverage falling below this threshold with an asterisk ($^*$) as statistically significant undercoverage.

\begin{itemize}
    \item \textbf{Sample Size on Risk ($n$, Table~\ref{tab:n_scaling_risk})}: 
    Nominal~$n$ fails across all evaluated sample sizes (collapsing to $0.647$ at $n=32$ and reaching only $0.913^*$ at $n=256$), demonstrating that sample size alone cannot cure dependence-induced undercoverage. At $n=32$, finite-sample estimation violates the moderate-deviation burn-in threshold $n/(\log n)^4 \ge \text{const}$ (Theorem~\ref{thm:kmm-dependent}), exhibiting expected transient undercoverage ($0.850$). Across all operational regimes ($n \ge 48$), GP~$n_{\mathrm{eff}}$ reliably maintains nominal coverage ($\ge 0.943$, reaching $0.997$ at $n=256$) while consistently yielding the sharpest valid intervals and superior Winkler Scores (up to $21\%$ narrower radius than Input~$n_{\mathrm{eff}}$).
    
    \item \textbf{Input Dependence on Risk ($\ell_x$, Table~\ref{tab:ablation-ellx})}: 
    As correlation length scales from $\ell_x = 1.5$ to $6.0$, raw covariate autocorrelation intensifies and Nominal~$n$ fails across all configurations (dropping to $0.767^*$ coverage at $\ell_x=4.0$ and remaining undercovered throughout). In contrast, GP~$n_{\mathrm{eff}}$ dynamically adapts to sequence persistence. Across all moderate-to-strong dependence regimes ($\ell_x \ge 2.0$), GP~$n_{\mathrm{eff}}$ achieves robust valid coverage ($\ge 0.947$) while consistently yielding approximately $21\%$ narrower intervals and superior Winkler Scores compared to the overly conservative Input~$n_{\mathrm{eff}}$. At weak dependence ($\ell_x=1.5$), where decorrelation is fast and empirical estimation adds relative noise, GP~$n_{\mathrm{eff}}$ exhibits a minor coverage dip ($0.923^*$), making Input~$n_{\mathrm{eff}}$ the preferred valid estimator ($0.953$), though GP~$n_{\mathrm{eff}}$ still achieves the lowest Winkler Score ($0.0159$).
    
    \item \textbf{Functional Smoothness on Risk ($h$, Table~\ref{tab:ablation-h})}: 
    This sweep empirically validates the Smoothness-Dependence Tradeoff (Section~\ref{sec:density-smoothness}). Because naive Input~$n_{\mathrm{eff}}$ accounts for raw coordinate autocorrelation, it produces an entirely static variance correction that ignores response smoothness, yielding overly conservative intervals. Meanwhile, Nominal~$n$ fails across all bandwidths (consistently under-covering below $0.918^*$, dropping to $0.827^*$). In stark contrast, GP~$n_{\mathrm{eff}}$ dynamically adapts to the functional length-scale $h$, identifying that the error surface decorrelates faster than raw inputs. Consequently, GP~$n_{\mathrm{eff}}$ dominates across all evaluated length-scales, maintaining robust coverage ($\ge 0.970$) while shrinking the confidence radius by $17\%$ to $28\%$ and achieving substantially lower Winkler Scores.
    
    \item \textbf{Clipping Bound on Risk ($B$, Table~\ref{tab:ablation-B})}: 
    The clipping parameter $B$ dictates the maximum allowable importance weight in KMM, controlling the bias-variance trade-off under domain overlap. As $B$ scales from $5.0$ to $30.0$, Nominal~$n$ fails across all thresholds (dropping to $0.780^*$ at $B=5.0$ and $0.790^*$ at $B=30.0$). In contrast, GP~$n_{\mathrm{eff}}$ demonstrates complete stability across all bounds, preserving a consistent variance penalty regardless of the clipping level. This empirically confirms the theoretical separation of density-ratio overlap ($B$) from sequence persistence in Theorem~\ref{thm:kmm-dependent}. Across the entire sweep, GP~$n_{\mathrm{eff}}$ maintains robust valid coverage ($\ge 0.947$, reaching $0.980$) while consistently shrinking interval radius by approximately $21\%$ and yielding substantially lower Winkler Scores than Input~$n_{\mathrm{eff}}$.
\end{itemize}

\begin{table*}[!htbp]
\centering
\caption{\textbf{Sample-size scaling of Target Risk intervals ($n \in \{32, \dots, 256\}$).} Bold highlights the best valid practical estimator (tightest radius and lowest Winkler Score among methods maintaining valid coverage). Asterisks ($^*$) indicate significant empirical undercoverage below the nominal 0.95 level ($< 0.925$). Strict Thm 1 is excluded from best-method highlighting due to worst-case conservatism.}
\label{tab:n_scaling_risk}
\vspace{1.5mm}
\resizebox{0.70\textwidth}{!}{
\begin{tabular}{@{}llccc@{}}
\toprule
$\bm{n}$ & \textbf{Method} & \textbf{Coverage} & \textbf{Radius} & \textbf{Winkler Score} $\downarrow$ \\
\midrule
\multirow{5}{*}{32} 
& Unweighted         & 0.717$^*$ & 0.0026 & 0.0201 \\
& Nominal $n$        & 0.647$^*$ & 0.0047 & 0.0394 \\
& Input $n_{\mathrm{eff}}$ & 0.893$^*$ & 0.0123 & 0.0325 \\
& GP $n_{\mathrm{eff}}$ (Ours) & 0.850$^*$ & 0.0098 & 0.0313 \\
\midrule
\multirow{5}{*}{48} 
& Unweighted         & 0.727$^*$ & 0.0022 & 0.0143 \\
& Nominal $n$        & 0.773$^*$ & 0.0048 & 0.0275 \\
& Input $n_{\mathrm{eff}}$ & 0.910$^*$ & 0.0128 & 0.0290 \\
& \textbf{GP $n_{\mathrm{eff}}$ (Ours)} & \textbf{0.943} & \textbf{0.0101} & \textbf{0.0260} \\
\midrule
\multirow{5}{*}{64} 
& Unweighted         & 0.773$^*$ & 0.0020 & 0.0124 \\
& Nominal $n$        & 0.773$^*$ & 0.0046 & 0.0218 \\
& Input $n_{\mathrm{eff}}$ & 0.980     & 0.0124 & 0.0258 \\
& \textbf{GP $n_{\mathrm{eff}}$ (Ours)} & \textbf{0.947} & \textbf{0.0098} & \textbf{0.0217} \\
\midrule
\multirow{5}{*}{96} 
& Unweighted         & 0.770$^*$ & 0.0016 & 0.0101 \\
& Nominal $n$        & 0.803$^*$ & 0.0043 & 0.0188 \\
& Input $n_{\mathrm{eff}}$ & 0.973     & 0.0117 & 0.0251 \\
& \textbf{GP $n_{\mathrm{eff}}$ (Ours)} & \textbf{0.957} & \textbf{0.0092} & \textbf{0.0213} \\
\midrule
\multirow{5}{*}{128} 
& Unweighted         & 0.750$^*$ & 0.0014 & 0.0083 \\
& Nominal $n$        & 0.823$^*$ & 0.0038 & 0.0155 \\
& Input $n_{\mathrm{eff}}$ & 0.983     & 0.0104 & 0.0220 \\
& \textbf{GP $n_{\mathrm{eff}}$ (Ours)} & \textbf{0.970} & \textbf{0.0082} & \textbf{0.0181} \\
\midrule
\multirow{5}{*}{192} 
& Unweighted         & 0.750$^*$ & 0.0011 & 0.0071 \\
& Nominal $n$        & 0.887$^*$ & 0.0033 & 0.0109 \\
& Input $n_{\mathrm{eff}}$ & 0.987     & 0.0091 & 0.0187 \\
& \textbf{GP $n_{\mathrm{eff}}$ (Ours)} & \textbf{0.980} & \textbf{0.0072} & \textbf{0.0152} \\
\midrule
\multirow{5}{*}{256} 
& Unweighted         & 0.743$^*$ & 0.0010 & 0.0062 \\
& Nominal $n$        & 0.913$^*$ & 0.0031 & 0.0076 \\
& Input $n_{\mathrm{eff}}$ & 1.000     & 0.0085 & 0.0170 \\
& \textbf{GP $n_{\mathrm{eff}}$ (Ours)} & \textbf{0.997} & \textbf{0.0067} & \textbf{0.0134} \\
\bottomrule
\end{tabular}}
\end{table*}

\begin{table*}[!htbp]
\centering
\caption{\textbf{Hyperparameter Ablation: Input Dependence ($\ell_x \in [1.5, 6.0]$) on Target Risk.} Asterisks ($^*$) indicate significant empirical undercoverage below the nominal 0.95 level ($< 0.925$). Bold highlights the best valid practical estimator (tightest radius and lowest Winkler Score among methods maintaining valid coverage). Strict Thm 1 is excluded from best-method highlighting due to worst-case conservatism.}
\label{tab:ablation-ellx}
\vspace{2mm}
\resizebox{0.70\textwidth}{!}{%
\begin{tabular}{@{}llccc@{}}
\toprule
$\bm{\ell_x}$ & \textbf{Method} & \textbf{Coverage} & \textbf{Radius} & \textbf{Winkler Score} $\downarrow$ \\
\midrule
1.5 
& Unweighted         & 0.843$^*$ & 0.0016 & 0.0061 \\
& Nominal $n$        & 0.840$^*$ & 0.0041 & 0.0161 \\
& \textbf{Input $n_{\mathrm{eff}}$} & \textbf{0.953} & \textbf{0.0078} & \textbf{0.0176} \\
& GP $n_{\mathrm{eff}}$ (Ours) & 0.923$^*$ & 0.0062 & 0.0159 \\
\midrule
2.0 
& Unweighted         & 0.843$^*$ & 0.0016 & 0.0068 \\
& Nominal $n$        & 0.870$^*$ & 0.0043 & 0.0143 \\
& Input $n_{\mathrm{eff}}$ & 0.987     & 0.0096 & 0.0204 \\
& \textbf{GP $n_{\mathrm{eff}}$ (Ours)} & \textbf{0.957} & \textbf{0.0076} & \textbf{0.0170} \\
& Strict Thm 1       & 1.000     & 4.1488 & 8.2975 \\
\midrule
3.0 
& Unweighted         & 0.767$^*$ & 0.0016 & 0.0091 \\
& Nominal $n$        & 0.803$^*$ & 0.0040 & 0.0184 \\
& Input $n_{\mathrm{eff}}$ & 0.977     & 0.0110 & 0.0226 \\
& \textbf{GP $n_{\mathrm{eff}}$ (Ours)} & \textbf{0.947} & \textbf{0.0086} & \textbf{0.0191} \\
\midrule
4.0 
& Unweighted         & 0.700$^*$ & 0.0016 & 0.0154 \\
& Nominal $n$        & 0.767$^*$ & 0.0039 & 0.0216 \\
& Input $n_{\mathrm{eff}}$ & 0.980     & 0.0120 & 0.0246 \\
& \textbf{GP $n_{\mathrm{eff}}$ (Ours)} & \textbf{0.950} & \textbf{0.0095} & \textbf{0.0210} \\
\midrule
5.0 
& Unweighted         & 0.680$^*$ & 0.0016 & 0.0164 \\
& Nominal $n$        & 0.773$^*$ & 0.0039 & 0.0244 \\
& Input $n_{\mathrm{eff}}$ & 0.980     & 0.0134 & 0.0282 \\
& \textbf{GP $n_{\mathrm{eff}}$ (Ours)} & \textbf{0.953} & \textbf{0.0106} & \textbf{0.0236} \\
\midrule
6.0 
& Unweighted         & 0.620$^*$ & 0.0016 & 0.0183 \\
& Nominal $n$        & 0.793$^*$ & 0.0042 & 0.0199 \\
& Input $n_{\mathrm{eff}}$ & 0.997     & 0.0161 & 0.0322 \\
& \textbf{GP $n_{\mathrm{eff}}$ (Ours)} & \textbf{0.987} & \textbf{0.0127} & \textbf{0.0259} \\
\bottomrule
\end{tabular}}
\end{table*}

\begin{table*}[!htbp]
\centering
\caption{\textbf{Hyperparameter Ablation: Functional Smoothness ($h \in [1.0, 3.0]$) on Target Risk.} Asterisks ($^*$) indicate significant empirical undercoverage below the nominal 0.95 level ($< 0.925$). Bold highlights the best valid practical estimator (tightest radius and lowest Winkler Score among methods maintaining valid coverage).}
\label{tab:ablation-h}
\vspace{2mm}
\resizebox{0.70\textwidth}{!}{%
\begin{tabular}{@{}llccc@{}}
\toprule
$\bm{h}$ & \textbf{Method} & \textbf{Coverage} & \textbf{Radius} & \textbf{Winkler Score} $\downarrow$ \\
\midrule
\multirow{4}{*}{1.00} 
& Unweighted         & 0.173$^*$ & 0.0023 & 0.0749 \\
& Nominal $n$        & 0.910$^*$ & 0.0062 & 0.0194 \\
& Input $n_{\mathrm{eff}}$ & 0.993     & 0.0167 & 0.0338 \\
& \textbf{GP $n_{\mathrm{eff}}$ (Ours)} & \textbf{0.983} & \textbf{0.0129} & \textbf{0.0271} \\
\midrule
\multirow{4}{*}{1.25} 
& Unweighted         & 0.047$^*$ & 0.0018 & 0.1597 \\
& Nominal $n$        & 0.917$^*$ & 0.0072 & 0.0183 \\
& Input $n_{\mathrm{eff}}$ & 0.993     & 0.0195 & 0.0395 \\
& \textbf{GP $n_{\mathrm{eff}}$ (Ours)} & \textbf{0.990} & \textbf{0.0156} & \textbf{0.0320} \\
\midrule
\multirow{4}{*}{1.50} 
& Unweighted         & 0.037$^*$ & 0.0015 & 0.1371 \\
& Nominal $n$        & 0.827$^*$ & 0.0057 & 0.0233 \\
& Input $n_{\mathrm{eff}}$ & 0.983     & 0.0154 & 0.0316 \\
& \textbf{GP $n_{\mathrm{eff}}$ (Ours)} & \textbf{0.980} & \textbf{0.0128} & \textbf{0.0272} \\
\midrule
\multirow{4}{*}{1.75} 
& Unweighted         & 0.137$^*$ & 0.0028 & 0.1304 \\
& Nominal $n$        & 0.907$^*$ & 0.0070 & 0.0188 \\
& Input $n_{\mathrm{eff}}$ & 0.997     & 0.0189 & 0.0378 \\
& \textbf{GP $n_{\mathrm{eff}}$ (Ours)} & \textbf{0.990} & \textbf{0.0136} & \textbf{0.0281} \\
\midrule
\multirow{4}{*}{2.00} 
& Unweighted         & 0.277$^*$ & 0.0017 & 0.0536 \\
& Nominal $n$        & 0.860$^*$ & 0.0057 & 0.0220 \\
& Input $n_{\mathrm{eff}}$ & 0.987     & 0.0154 & 0.0323 \\
& \textbf{GP $n_{\mathrm{eff}}$ (Ours)} & \textbf{0.970} & \textbf{0.0116} & \textbf{0.0257} \\
\midrule
\multirow{4}{*}{3.00} 
& Unweighted         & 0.490$^*$ & 0.0018 & 0.0292 \\
& Nominal $n$        & 0.877$^*$ & 0.0057 & 0.0187 \\
& Input $n_{\mathrm{eff}}$ & 0.990     & 0.0153 & 0.0310 \\
& \textbf{GP $n_{\mathrm{eff}}$ (Ours)} & \textbf{0.983} & \textbf{0.0113} & \textbf{0.0238} \\
\bottomrule
\end{tabular}}
\end{table*}

\begin{table*}[!htbp]
\centering
\caption{\textbf{Hyperparameter Ablation: Clipping Bound ($B \in [5.0, 30.0]$) on Target Risk.} Asterisks ($^*$) indicate significant empirical undercoverage below the nominal 0.95 level ($< 0.925$). Bold highlights the best valid practical estimator (tightest radius and lowest Winkler Score among methods maintaining valid coverage).}
\label{tab:ablation-B}
\vspace{2mm}
\resizebox{0.70\textwidth}{!}{%
\begin{tabular}{@{}llccc@{}}
\toprule
$\bm{B}$ & \textbf{Method} & \textbf{Coverage} & \textbf{Radius} & \textbf{Winkler Score} $\downarrow$ \\
\midrule
\multirow{4}{*}{5.0} 
& Unweighted         & 0.767$^*$ & 0.0016 & 0.0100 \\
& Nominal $n$        & 0.780$^*$ & 0.0030 & 0.0154 \\
& Input $n_{\mathrm{eff}}$ & 0.987     & 0.0082 & 0.0168 \\
& \textbf{GP $n_{\mathrm{eff}}$ (Ours)} & \textbf{0.963} & \textbf{0.0065} & \textbf{0.0141} \\
\midrule
\multirow{4}{*}{7.5} 
& Unweighted         & 0.767$^*$ & 0.0016 & 0.0100 \\
& Nominal $n$        & 0.803$^*$ & 0.0035 & 0.0147 \\
& Input $n_{\mathrm{eff}}$ & 0.993     & 0.0095 & 0.0192 \\
& \textbf{GP $n_{\mathrm{eff}}$ (Ours)} & \textbf{0.980} & \textbf{0.0075} & \textbf{0.0156} \\
\midrule
\multirow{4}{*}{10.0} 
& Unweighted         & 0.767$^*$ & 0.0016 & 0.0100 \\
& Nominal $n$        & 0.837$^*$ & 0.0038 & 0.0150 \\
& Input $n_{\mathrm{eff}}$ & 0.990     & 0.0103 & 0.0209 \\
& \textbf{GP $n_{\mathrm{eff}}$ (Ours)} & \textbf{0.970} & \textbf{0.0082} & \textbf{0.0172} \\
\midrule
\multirow{4}{*}{15.0} 
& Unweighted         & 0.767$^*$ & 0.0016 & 0.0100 \\
& Nominal $n$        & 0.803$^*$ & 0.0042 & 0.0171 \\
& Input $n_{\mathrm{eff}}$ & 0.987     & 0.0112 & 0.0231 \\
& \textbf{GP $n_{\mathrm{eff}}$ (Ours)} & \textbf{0.967} & \textbf{0.0089} & \textbf{0.0191} \\
\midrule
\multirow{4}{*}{20.0} 
& Unweighted         & 0.767$^*$ & 0.0016 & 0.0100 \\
& Nominal $n$        & 0.797$^*$ & 0.0043 & 0.0184 \\
& Input $n_{\mathrm{eff}}$ & 0.987     & 0.0118 & 0.0245 \\
& \textbf{GP $n_{\mathrm{eff}}$ (Ours)} & \textbf{0.960} & \textbf{0.0093} & \textbf{0.0203} \\
\midrule
\multirow{4}{*}{30.0} 
& Unweighted         & 0.767$^*$ & 0.0016 & 0.0100 \\
& Nominal $n$        & 0.790$^*$ & 0.0045 & 0.0207 \\
& Input $n_{\mathrm{eff}}$ & 0.980     & 0.0122 & 0.0254 \\
& \textbf{GP $n_{\mathrm{eff}}$ (Ours)} & \textbf{0.947} & \textbf{0.0096} & \textbf{0.0214} \\
\bottomrule
\end{tabular}}
\end{table*}

\subsection{Kernel Family and Roughness Ablation (Mat\'ern Processes)}
\label{app:matern}

The derivation of our closed-form GP $n_{\mathrm{eff}}$ formula relies on a Squared-Exponential (RBF) covariance assumption, which models an infinitely differentiable process. A theoretical concern is whether applying this RBF-derived correction yields overly optimistic (under-covered) intervals if the true underlying data-generating process is physically rough or non-differentiable (e.g., admitting a Mat\'ern covariance structure).

To stress-test this, we evaluate KMM over data generated from exact spectral Mat\'ern realizations. The Mat\'ern family is parameterized by a smoothness parameter $\nu$. While $\nu \to \infty$ recovers the RBF kernel, smaller values ($\nu = 5/2, 3/2, 1/2$) yield increasingly rough, non-differentiable processes. Mat\'ern-1/2, for example, represents a continuous but nowhere-differentiable path (equivalent to an Exponential covariance). In this experiment, our GP $n_{\mathrm{eff}}$ estimator continues to assume an RBF geometry, deliberately inducing a smoothness misspecification.

Table~\ref{tab:matern-ablation} demonstrates the robustness of the RBF-derived surrogate evaluated over Target Risk. Across all Mat\'ern kernels, GP $n_{\mathrm{eff}}$ maintains substantially lower Winkler Scores and tighter intervals than the conservative Input $n_{\mathrm{eff}}$. For Mat\'ern-5/2, the empirical coverage of our method dips slightly to $0.913$. However, the theoretical benchmark \textit{Oracle i.i.d.} (which evaluates risk on genuinely independent target samples) drops to an identical $0.917$. This confirms that this mild degradation reflects the inherent sample complexity of evaluating small samples on this specific manifold, rather than a failure of the effective sample size approximation. 

Even on the extremely rough, nowhere-differentiable Mat\'ern-1/2 process, our misspecified RBF formula maintains robust valid coverage ($0.993$) while shrinking the confidence radius by approximately $27\%$ relative to the raw input correction.

\begin{table*}[!htbp]
\centering
\caption{\textbf{Kernel Type Ablation (Roughness Stress Test) on Target Risk.} Asterisks ($^*$) denote significant empirical undercoverage below the nominal 0.95 level ($< 0.925$). GP $n_{\mathrm{eff}}$ maintains valid interval coverage even when the true generating process is non-differentiable, consistently outperforming Input $n_{\mathrm{eff}}$ in Winkler Score (WS). Note that the slight undercoverage on Mat\'ern-5/2 is matched by the Oracle i.i.d. benchmark, indicating inherent finite-sample variance rather than a failure of the GP correction.}
\label{tab:matern-ablation}
\vspace{2mm}
\resizebox{0.78\textwidth}{!}{%
\begin{tabular}{@{}llccc@{}}
\toprule
\textbf{Data Kernel} & \textbf{Method} & \textbf{Coverage} & \textbf{Radius} & \textbf{WS} $\downarrow$ \\
\midrule
\multirow{5}{*}{RBF (Smooth)} 
& Unweighted               & 0.110$^*$ & 0.0022 & 0.113 \\
& Nominal $n$              & 0.903$^*$ & 0.0076 & 0.023 \\
& Input $n_{\mathrm{eff}}$ & 0.993\phantom{$^*$} & 0.0206 & 0.042 \\
& \textbf{GP $\bm{n_{\mathrm{eff}}}$ (Ours)} & \textbf{0.987}\phantom{$^*$} & \textbf{0.0147} & \textbf{0.031} \\
\midrule
\multirow{5}{*}{Mat\'ern-5/2} 
& Unweighted               & 0.623$^*$ & 0.0019 & 0.019 \\
& Nominal $n$              & 0.850$^*$ & 0.0059 & 0.029 \\
& Input $n_{\mathrm{eff}}$ & 0.960\phantom{$^*$} & 0.0160 & 0.036 \\
& \textbf{GP $\bm{n_{\mathrm{eff}}}$ (Ours)} & \textbf{0.913}$^*$ & \textbf{0.0095} & \textbf{0.028} \\
& \textit{Oracle i.i.d.}   & \textit{0.917}$^*$ & \textit{0.0062} & \textit{0.018} \\
\midrule
\multirow{5}{*}{Mat\'ern-3/2} 
& Unweighted               & 0.080$^*$ & 0.0028 & 0.164 \\
& Nominal $n$              & 0.900$^*$ & 0.0065 & 0.020 \\
& Input $n_{\mathrm{eff}}$ & 1.000\phantom{$^*$} & 0.0175 & 0.035 \\
& \textbf{GP $\bm{n_{\mathrm{eff}}}$ (Ours)} & \textbf{0.990}\phantom{$^*$} & \textbf{0.0128} & \textbf{0.026} \\
\midrule
\multirow{5}{*}{Mat\'ern-1/2 (Rough)} 
& Unweighted               & 0.057$^*$ & 0.0023 & 0.132 \\
& Nominal $n$              & 0.923$^*$ & 0.0068 & 0.019 \\
& Input $n_{\mathrm{eff}}$ & 1.000\phantom{$^*$} & 0.0184 & 0.037 \\
& \textbf{GP $\bm{n_{\mathrm{eff}}}$ (Ours)} & \textbf{0.993}\phantom{$^*$} & \textbf{0.0134} & \textbf{0.027} \\
\bottomrule
\end{tabular}}
\end{table*}

\subsection{Robustness to Non-Stationary Covariates}
\label{app:nonstat}

A core structural assumption in standard mixing theory is the strict stationarity of the covariate sampling process $\mathbf{X}_t$. In realistic ecological or spatio-temporal monitoring, however, the effective correlation length $\ell_x(t)$ frequently fluctuates due to non-uniform sampling velocity, terrain transitions, or seasonal variations.

To stress-test our framework against structural non-stationarity, we evaluate KMM over covariate sequences generated via a non-stationary Gibbs covariance kernel \citep{gibbs1998bayesian}. The Gibbs kernel generalizes the stationary RBF covariance by allowing the local correlation length $\ell_x(t)$ to vary continuously across sequence indices:
\begin{equation}
\operatorname{Cov}(X_s, X_t) = \sigma_0^2 \sqrt{\frac{2\ell_x(s)\ell_x(t)}{\ell_x(s)^2 + \ell_x(t)^2}} \exp\left( - \frac{|s-t|^2}{\ell_x(s)^2 + \ell_x(t)^2} \right).
\end{equation}
We evaluate three sampling regimes with sample size $n_{\mathrm{tr}} = 128$:
\begin{enumerate}
    \item \textbf{Stationary Baseline}: Fixed persistence length $\ell_x(t) \equiv 2.5$.
    \item \textbf{Abrupt Switch}: A sharp structural regime break where the sequence abruptly shifts from dense oversampling ($\ell_x = 0.5$ for $t \le n_{\mathrm{tr}}/2$) to long-range persistent dependence ($\ell_x = 4.5$ for $t > n_{\mathrm{tr}}/2$), with sequence average $\bar{\ell}_x = 2.5$.
    \item \textbf{Smooth Wave}: A continuous harmonic oscillation simulating periodic variations in sampling density: $\ell_x(t) = 2.5 + 2.0\sin(2\pi t / n_{\mathrm{tr}})$, with average $\bar{\ell}_x = 2.5$.
\end{enumerate}

Throughout all non-stationary trials, our estimator deliberately remains blind to the time-varying profile $\ell_x(t)$, simply plugging the sequence average $\bar{\ell}_x = 2.5$ into the stationary GP $n_{\mathrm{eff}}$ formula. 

As reported in Table~\ref{tab:nonstationary-ablation}, because KMM evaluates risk as a trajectory-wide integral, local fluctuations in dependence length are safely bounded by the sequence-averaged length scale. While Nominal $n$ systematically under-covers on Target Risk across all configurations (collapsing to $0.880^*$ under abrupt switches), our GP $n_{\mathrm{eff}}$ maintains robust valid coverage ($\ge 0.983$) across both stationary and non-stationary profiles. Furthermore, GP $n_{\mathrm{eff}}$ achieves a $17.5\%$ tighter confidence radius and superior Winkler Scores compared to the overly conservative Input $n_{\mathrm{eff}}$, confirming that the proposed correction remains safe even when strict stationarity is violated.

\begin{table*}[!htbp]
\centering
\caption{\textbf{Robustness to Non-Stationary Covariates on Target Risk.} Asterisks ($^*$) indicate significant empirical undercoverage below the nominal 0.95 level ($< 0.925$). Covariates are generated via a non-stationary Gibbs kernel with mean dependence $\bar{\ell}_x = 2.5$. Plugging the trajectory-average correlation length into our stationary formula maintains valid coverage ($\ge 0.983$) and yields the lowest Winkler Score (WS) among valid estimators across all non-stationary regimes.}
\label{tab:nonstationary-ablation}
\vspace{2mm}
\resizebox{0.78\textwidth}{!}{%
\begin{tabular}{@{}llccc@{}}
\toprule
\textbf{Dependence Profile} & \textbf{Method} & \textbf{Coverage} & \textbf{Radius} & \textbf{WS} $\downarrow$ \\
\midrule
\multirow{5}{*}{Stationary} 
& Unweighted               & 0.700$^*$ & 0.0020 & 0.013 \\
& Nominal $n$              & 0.900$^*$ & 0.0048 & 0.014 \\
& Input $n_{\mathrm{eff}}$ & 0.993\phantom{$^*$} & 0.0120 & 0.024 \\
& \textbf{GP $\bm{n_{\mathrm{eff}}}$ (Ours)} & \textbf{0.990}\phantom{$^*$} & \textbf{0.0099} & \textbf{0.020} \\
\midrule
\multirow{5}{*}{Abrupt Switch} 
& Unweighted               & 0.717$^*$ & 0.0020 & 0.014 \\
& Nominal $n$              & 0.880$^*$ & 0.0046 & 0.014 \\
& Input $n_{\mathrm{eff}}$ & 0.997\phantom{$^*$} & 0.0114 & 0.023 \\
& \textbf{GP $\bm{n_{\mathrm{eff}}}$ (Ours)} & \textbf{0.990}\phantom{$^*$} & \textbf{0.0094} & \textbf{0.019} \\
\midrule
\multirow{5}{*}{Smooth Wave} 
& Unweighted               & 0.750$^*$ & 0.0020 & 0.011 \\
& Nominal $n$              & 0.893$^*$ & 0.0046 & 0.016 \\
& Input $n_{\mathrm{eff}}$ & 0.997\phantom{$^*$} & 0.0115 & 0.023 \\
& \textbf{GP $\bm{n_{\mathrm{eff}}}$ (Ours)} & \textbf{0.983}\phantom{$^*$} & \textbf{0.0095} & \textbf{0.020} \\
\bottomrule
\end{tabular}}
\end{table*}

\section{Real Data Experiments Details}
\label{app:real_data_details}

\subsection{Semi-Synthetic Overlap Stress Test on California Housing}
\label{app:california_overlap}

The controlled synthetic experiments isolate dependence and evaluate the dependence-induced effective sample size developed in the main text. We use a separate semi-synthetic experiment to isolate a different component of the problem: deterioration of source--target overlap. Real covariates and geographic coordinates are retained from the California Housing dataset, while outcomes are generated from a known conditional mechanism shared by source and target. This allows overlap to be varied without simultaneously introducing conditional shift.

Let $q_0(x)$ denote the empirical distribution over the California Housing covariates. We define a standardized overlap score:
\[
g(x) = \operatorname{std}\!\left(0.75\,X_{\mathrm{MedInc}} + 0.45\,X_{\mathrm{Latitude}} + 0.25\,X_{\mathrm{AveRooms}}\right),
\]
which is truncated to $[-3,3]$. Source and target distributions are constructed by symmetric exponential tilting:
\begin{equation}
q_S^{(\gamma)}(x) \propto q_0(x)\exp\{-\gamma g(x)/2\}, \qquad q_T^{(\gamma)}(x) \propto q_0(x)\exp\{+\gamma g(x)/2\}.
\label{eq:california_tilting}
\end{equation}
Consequently,
\begin{equation}
\frac{q_T^{(\gamma)}(x)}{q_S^{(\gamma)}(x)} \propto \exp\{\gamma g(x)\},
\label{eq:california_ratio}
\end{equation}
so the population density ratio is analytically known.

We consider shift severities $\gamma \in \{0, 0.5, 1, 1.5, 2, 3\}$. At $\gamma=0$, source and target coincide; increasing $\gamma$ progressively reduces overlap. For each value of $\gamma$, we draw $n_S = n_T = 256$ observations independently with replacement from the corresponding tilted empirical distributions. Results are averaged over 100 repetitions.

Let the standardized California Housing predictors be denoted by $X$. We construct:
\[
\mu(X) = 0.2 + 0.6\,\sigma\{\eta(X)\},
\]
where $\sigma$ is the logistic function and
\begin{align}
\eta(X) ={}& 0.75X_{\mathrm{MedInc}} + 0.25\sin(X_{\mathrm{HouseAge}}) + 0.20\tanh(X_{\mathrm{AveRooms}}) \nonumber\\
&- 0.15\tanh(X_{\mathrm{AveOccup}}) + 0.20\sin(X_{\mathrm{Latitude}}) + 0.20\cos(X_{\mathrm{Longitude}}) \nonumber\\
&+ 0.12X_{\mathrm{MedInc}}X_{\mathrm{AveRooms}}.
\end{align}
The observed response is
\[
Y(s) = \mu(X(s)) + \varepsilon(s),
\]
where $\varepsilon(s)$ is a bounded spatial random field generated using random Fourier features:
\[
\varepsilon(s) = 0.12\tanh\{Z(s)\}.
\]
The Fourier coefficients are independently sampled from symmetric zero-mean Gaussian distributions in each repetition. Hence the residual process has zero expectation across realizations and is spatially correlated within a realization. Since $\mu(X) \in [0.2, 0.8]$ and $\varepsilon(s) \in [-0.12, 0.12]$, the response remains bounded in $[0.08, 0.92]$ by construction. Across the final experiment, the realized response range was $[0.091, 0.917]$.

The same conditional mechanism is used for source and target for all $\gamma$. Thus, the experiment changes covariate overlap while keeping the conditional response mechanism invariant. The target estimand is the population target mean:
\begin{equation}
\mu_T^{(\gamma)} = \sum_x q_T^{(\gamma)}(x)\mu(x),
\label{eq:california_target}
\end{equation}
which is known from the semi-synthetic construction.

\paragraph{Estimators and KMM Configuration}
We compare three estimators: the unweighted source mean, self-normalized oracle importance weighting using the known ratio in Eq.~\eqref{eq:california_ratio}, and KMM using source and unlabeled target covariates. The KMM weights satisfy:
\[
0 \leq w_i \leq B, \qquad \sum_{i=1}^{n_S} w_i = n_S,
\]
with clipping bound $B = 10$. We use an RBF kernel with bandwidth $h = 8$ and ridge regularization $\lambda = 10^{-2}$.

These KMM hyperparameters were selected in a separate no-shift sanity check at $\gamma=0$, using only covariate balance and weight stability and without using outcomes. We evaluated $h \in \{1, 2, 4, 8\}$ and $\lambda \in \{10^{-4}, 10^{-3}, 10^{-2}\}$. The selected configuration gave mean weight ESS $215.9$ out of 256, mean maximum weight $3.16$, no weights at the upper bound, and reduced empirical MMD$^2$ from $5.81 \times 10^{-4}$ to $2.1 \times 10^{-5}$ in the sanity experiment.

\paragraph{Overlap and Weight Diagnostics}
We quantify empirical source--target discrepancy using MMD$^2$ and weight concentration using:
\begin{equation}
\operatorname{ESS}_w = \frac{\left(\sum_i w_i\right)^2}{\sum_i w_i^2}.
\label{eq:california_weight_ess}
\end{equation}
This quantity is a diagnostic of importance-weight concentration and is distinct from the dependence-induced effective sample size $n_{\mathrm{eff}}$ central to our theoretical results.

\begin{figure}[!ht]
\centering
\includegraphics[width=\columnwidth]{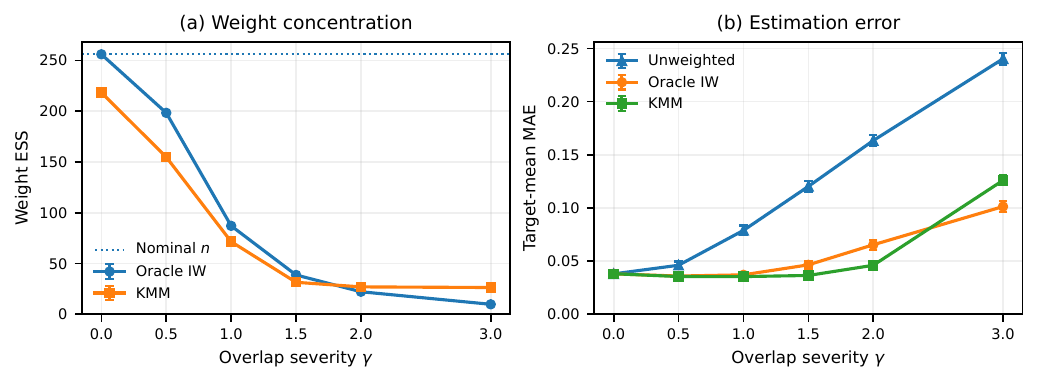}
\caption{\textbf{Controlled Overlap Stress Test on California Housing Covariates.} Increasing $\gamma$ progressively separates source and target distributions. \textbf{Left:} Oracle and KMM weight ESS decrease as overlap deteriorates. \textbf{Right:} KMM controls target-mean error under moderate shift, but its bounded weights cannot fully reproduce the density ratio under severe overlap loss. Weight ESS measures weight concentration and is distinct from the dependence-induced $n_{\mathrm{eff}}$. Error bars show one standard error across 100 repetitions.}
\label{fig:california_overlap}
\end{figure}

\begin{table}[!ht]
\centering
\caption{\textbf{Semi-Synthetic Overlap Stress Test on California Housing Covariates.} Increasing $\gamma$ separates source and target distributions. ESS denotes weight effective sample size ($\operatorname{ESS}_w$), not the dependence-induced $n_{\mathrm{eff}}$. MAE is computed over 100 repetitions.}
\label{tab:california_overlap}
\vspace{1.5mm}
\resizebox{\columnwidth}{!}{%
\begin{tabular}{@{}lccccccc@{}}
\toprule
$\bm{\gamma}$ & \textbf{MMD$\bm{^2}$} & \textbf{Oracle ESS} & \textbf{KMM ESS} & \textbf{Oracle Max $\bm{w}$} & \textbf{Unweighted MAE} & \textbf{Oracle MAE} & \textbf{KMM MAE} \\
\midrule
0.0 & 0.0007 & 256.0 & 218.3 & 1.0  & 0.0380 & 0.0380 & 0.0381 \\
0.5 & 0.0048 & 198.3 & 154.8 & 4.3  & 0.0461 & 0.0359 & \textbf{0.0355} \\
1.0 & 0.0176 & 87.0  & 71.5  & 15.5 & 0.0789 & 0.0371 & \textbf{0.0354} \\
1.5 & 0.0393 & 38.7  & 31.5  & 36.6 & 0.1203 & 0.0465 & \textbf{0.0365} \\
2.0 & 0.0704 & 22.1  & 26.9  & 58.9 & 0.1634 & 0.0653 & \textbf{0.0460} \\
3.0 & 0.1549 & 9.8   & 26.3  & 92.8 & 0.2403 & \textbf{0.1013} & 0.1260 \\
\bottomrule
\end{tabular}%
}
\end{table}

\paragraph{Results and Discussion}
The construction produces the intended progressive deterioration in overlap (Figure~\ref{fig:california_overlap} and Table~\ref{tab:california_overlap}). Mean empirical MMD$^2$ increases from $0.0007$ at $\gamma=0$ to $0.1549$ at $\gamma=3$, while oracle weight ESS falls from $256$ to $9.8$. KMM weight ESS similarly decreases from $218.3$ to $26.3$. Across all replicates, $\gamma$ is strongly positively associated with MMD$^2$ (Spearman $\rho = 0.986$) and negatively associated with both oracle and KMM weight ESS ($\rho = -0.937$ and $-0.975$, respectively).

At $\gamma=0$, all three estimators have essentially identical error: MAE is $0.0380$ for the unweighted and oracle estimators and $0.0381$ for KMM. This confirms that weighting does not introduce unnecessary distortion when source and target coincide.

As overlap deteriorates, the unweighted estimator rapidly degrades, whereas reweighting substantially reduces target-mean error. At $\gamma=2$, for example, MAE is $0.1634$ without weighting, $0.0653$ with oracle weighting, and $0.0460$ with KMM, yielding an MAE reduction of approximately $72\%$ relative to the unweighted baseline.

The severe-overlap regime illustrates the limitation of interpreting weight ESS as a measure of successful transfer. At $\gamma=3$, the mean maximum oracle weight is $92.8$ and oracle weight ESS falls to $9.8$. KMM is constrained by $B=10$ and consequently retains a larger ESS of $26.3$. Nevertheless, its MAE is $0.1260$, compared with $0.1013$ for oracle weighting. The larger KMM ESS reflects more stable weight concentration, not better approximation of the target density ratio. Bounding the weights limits variance amplification but can leave residual reweighting bias when adequate overlap would require more extreme weights. Consistent with this pattern, KMM weight ESS is negatively associated with KMM absolute error across replicates (Spearman $\rho = -0.391$), but this association does not imply that larger ESS alone guarantees smaller target error.

\paragraph{The Role of Clipping Bound $B$ in Theorem 1}
The sweep exposes a direct balance--stability trade-off. When $\gamma=0$, $B^*=1$ and increasing the admissible weight bound has essentially no effect. At $\gamma=1$, the empirical results similarly stabilize once $B$ approaches $B^* \approx 19.85$.

Under stronger overlap loss, however, the required population ratio becomes much larger than the weight bounds considered. At $\gamma=2$, increasing $B$ from $5$ to $20$ reduces postweighting MMD$^2$ from $0.0162$ to $0.0047$ and reduces bias magnitude from $0.0611$ to $0.0013$. This improvement in balance is accompanied by a decrease in weight ESS from $52.1$ to $15.1$. Increasing $B$ further provides little additional RMSE improvement, illustrating the finite-sample bias--variance trade-off induced by the weight constraint (Figure 6).

The same phenomenon is more pronounced at $\gamma=3$, where $B^*=4739.9$ and even $B=100$ represents only approximately $2.1\%$ of the population ratio bound. Increasing $B$ from $5$ to $100$ reduces post-weighting MMD$^2$ from $0.0768$ to $0.0288$ and bias magnitude from $0.1533$ to $0.0465$, while weight ESS falls from $51.8$ to $4.0$. Thus, a larger weight ESS does not by itself indicate successful target correction: restrictive bounds can retain a comparatively large ESS precisely because they prevent the weights required to remove source--target imbalance.

These results clarify that the bounded-density-ratio assumption in Theorem~1 is substantive rather than merely technical. When $B < B^*(\gamma)$, the admissible KMM weights cannot represent the population density ratio, leaving residual RKHS imbalance. Relaxing the bound reduces this imbalance and its associated bias, but simultaneously increases weight concentration and sampling variability.

\begin{figure*}[ht!]
\centering
\includegraphics[width=0.3\textwidth]{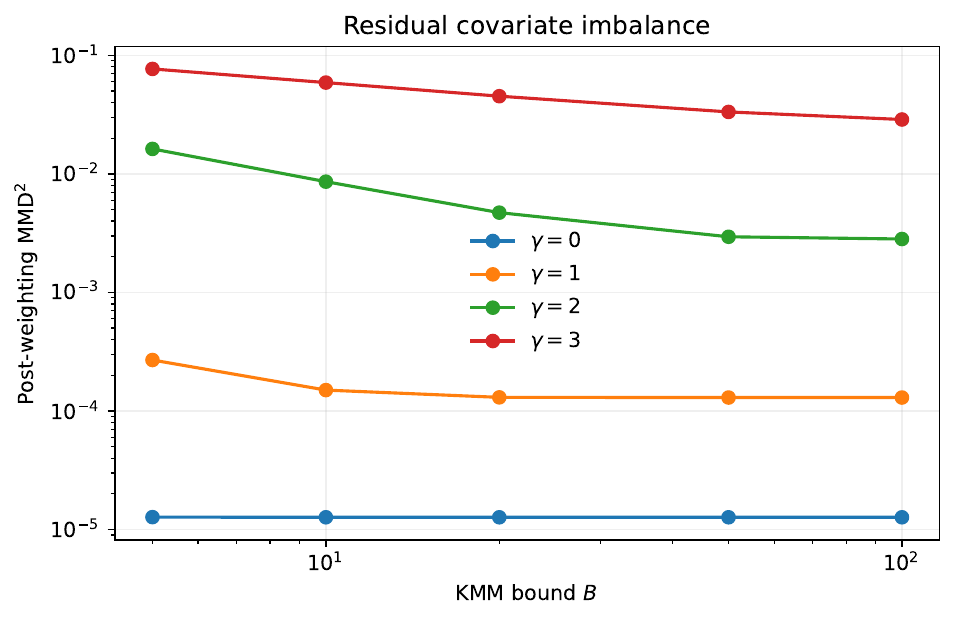}
\hfill
\includegraphics[width=0.3\textwidth]{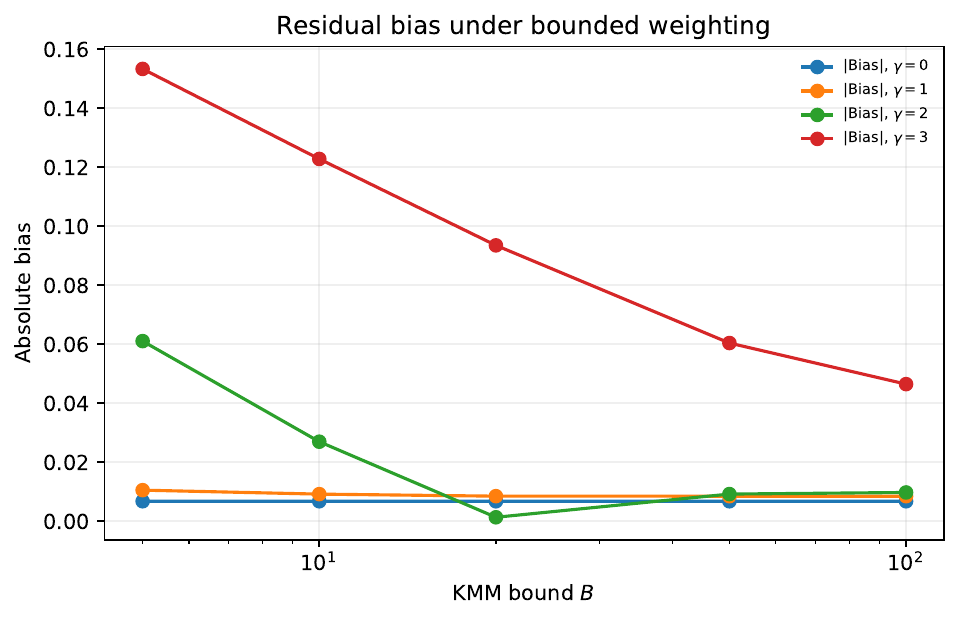}
\hfill
\includegraphics[width=0.3\textwidth]{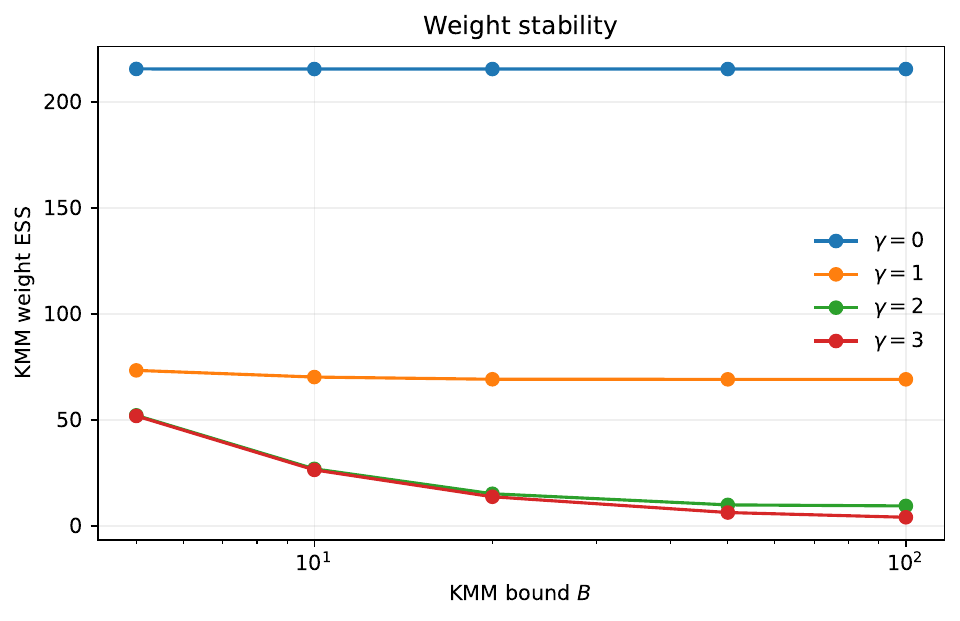}

\caption{
\textbf{Effect of the KMM clipping bound $B$ under deteriorating overlap.}
Increasing $B$ allows KMM to achieve better source--target balance and reduces residual estimation bias, but at the cost of increasingly
concentrated weights and lower weight ESS.
When the population density-ratio bound $B^*$ is moderate ($\gamma=1$), the results stabilize once $B$ reaches $B^*$.Under severe overlap loss ($\gamma=2,3$), $B^*$ is much larger than
the considered clipping levels, exposing a pronounced balance--stability trade-off.
Weight ESS here measures weight concentration and is distinct from the dependence-induced $n_{\mathrm{eff}}$ used elsewhere in the paper.}
\label{fig:california_B_sweep}
\end{figure*}

\subsection{Three-Domain Real-Data Benchmark: Point Estimation}
\label{app:real_point}

In Section~\ref{sec:real_benchmark}, our primary focus is on confidence interval calibration under data dependence. For completeness, Table~\ref{tab:real_point} reports point estimation metrics across the three real-data monitoring domains: NOAA weather stations, PM$_{2.5}$ monitoring stations, and EPA ozone monitoring stations.

We compare the unweighted source mean, logistic importance weighting (Log-IW), and KMM for two distinct estimands:
\begin{enumerate}
    \item \textbf{Target Response Mean:} $\mu_T = \mathbb{E}_T[Y]$.
    \item \textbf{Prediction Squared-Error Risk:} $R_T(f) = \mathbb{E}_T[(Y - f(X))^2]$ for the frozen random-forest predictor $f$.
\end{enumerate}

\paragraph{Datasets and response scaling.}
All responses are clipped and normalized to $[0,1]$; consequently, the reported errors are in normalized units, not physical units. NOAA uses daily station observations from 2025, with response
$Y=\operatorname{clip}((\mathrm{TMAX}+30)/80,0,1)$, where
$\mathrm{TMAX}$ is measured in degrees Celsius. Predictors are minimum temperature, precipitation, and sine/cosine encodings of day of year. The Beijing PM$_{2.5}$ dataset uses
$Y=\operatorname{clip}(\mathrm{PM}_{2.5}/500,0,1)$, with concentration measured in $\mu\mathrm{g}/\mathrm{m}^3$.
Predictors are temperature, pressure, dew point, rainfall, wind speed, SO$_2$, NO$_2$, CO, O$_3$, and sine/cosine encodings of day of year.
EPA Ozone uses 2022 daily measurements of parameter 44201: the daily maximum eight-hour ozone concentration, normalized as
$Y=\operatorname{clip}(\mathrm{O}_3/0.2,0,1)$, with ozone in ppm. Temperature and wind measurements are joined by monitoring-site
identifier and calendar date; predictors are temperature, wind speed, and sine/cosine encodings of day of year.

\paragraph{Predictor and weighting methods.}
The random forest is trained separately from the evaluation tasks and frozen thereafter. It uses 100 trees, a minimum leaf size of 10, all predictors as candidate features, and random seed
20261005. Missing predictor values are handled by median imputation in the prediction pipeline, and predictions are clipped to $[0,1]$.
KMM uses an RBF kernel with a pooled source--target, covariate-only standardization and median-distance bandwidth. Weights satisfy $0\leq w_i\leq 10$ and $\sum_i w_i=n$.
Logistic importance weighting estimates density ratios by classifying source versus target covariates; ratios are upper-capped at 10 and self-normalized for point estimation. Target outcomes are not used to train the predictor or fit the importance weights.

\paragraph{Point-estimation versus confidence interval tasks.}
The point-estimation results in Table~\ref{tab:real_point} utilize a fundamentally different task construction compared to the rolling-window interval evaluation detailed later. Specifically, point estimation is computed over full-series or full-quarter transfers: 141 NOAA transfer tasks, 48 PM$_{2.5}$ station--quarter tasks, and 50 EPA Ozone full-series transfer tasks. For PM$_{2.5}$, the predictor is trained on 2013--2014 data, with source observations drawn from 2015 and target observations from the corresponding quarters of 2016. In contrast, the confidence interval evaluation (described in Appendix~\ref{app:real_intervals}) relies on a strict rolling-window protocol over 250 independent trials of fixed length $n$. Therefore, the point-estimation errors reported here reflect the overall performance on broad transfer tasks and do not directly correspond to the specific 250 finite-sample segments utilized in the interval calibration experiments.

\paragraph{Evaluation reference.}
Population target means and risks are unknown.
For evaluation, they are replaced by the empirical response mean and empirical squared-error risk on the held-out target sample. MAE is the unweighted average of task-level absolute
estimation errors. Target outcomes are accessed only for final scoring. Tasks can share stations and evaluation periods, so task counts do not represent independent Monte Carlo repetitions.

\begin{table}[!ht]
\centering
\caption{\textbf{Point Estimation on the Three-Domain Real-Data Benchmark.} Mean Absolute Error (MAE) is reported for target response mean $\mu_T = \mathbb{E}_T[Y]$ and squared-error risk $R_T(f) = \mathbb{E}_T[(Y - f(X))^2]$. Bold indicates the lowest MAE within each dataset--estimand pair.}
\label{tab:real_point}
\vspace{1.5mm}
\resizebox{\columnwidth}{!}{%
\begin{tabular}{@{}llccc@{}}
\toprule
\textbf{Data} & \textbf{Estimand} & \textbf{Unweighted} & \textbf{Log-IW} & \textbf{KMM} \\
\midrule
NOAA       & Target Mean ($\mu_T$)  & 0.0503  & 0.0389          & \textbf{0.0379} \\
PM$_{2.5}$ & Target Mean ($\mu_T$)  & 0.0208  & 0.0270          & \textbf{0.0155} \\
EPA Ozone  & Target Mean ($\mu_T$)  & 0.0290  & 0.0245          & \textbf{0.0239} \\
\midrule
NOAA       & Target Risk ($R_T(f)$) & 0.00209 & \textbf{0.00205} & 0.00213 \\
PM$_{2.5}$ & Target Risk ($R_T(f)$) & 0.00214 & \textbf{0.00205} & 0.00244 \\
EPA Ozone  & Target Risk ($R_T(f)$) & 0.00137 & \textbf{0.00109} & 0.00120 \\
\bottomrule
\end{tabular}%
}
\end{table}

\paragraph{Discussion of Point Estimation Results.}
As shown in Table~\ref{tab:real_point}, KMM achieves the lowest target-mean MAE across all three domains, reducing the unweighted error from $0.0503$ to $0.0379$ for NOAA, from $0.0208$ to $0.0155$ for PM$_{2.5}$, and from $0.0290$ to $0.0239$ for EPA Ozone. 
However, the pattern differs for prediction risk: Logistic-IW achieves the lowest risk MAE across all three domains, while KMM does not consistently outperform the unweighted baseline. These results highlight that achieving improved covariate balance does not guarantee uniform improvements in point estimation across all observational transfer tasks.

\subsection{Three-Domain Real-Data Benchmark: Confidence Interval Protocol}
\label{app:real_intervals}

The interval calibration experiment addresses a distinct question: how accounting for dependent data generation alters interval coverage and width around a fixed KMM estimate. GP-based intervals improve empirical coverage relative to nominal intervals, although coverage remains below the nominal 95\% level for NOAA and EPA Ozone. We therefore interpret these results as a real-data stress test of the dependence correction rather than as verification of the theorem's assumptions or finite-sample coverage guarantees.

\paragraph{Rolling-Window Evaluation Protocol.}
For the confidence interval calibration experiment in Section~\ref{sec:real_benchmark}, we adopted a rigorous Monte Carlo rolling-window protocol to ensure statistically robust coverage estimates. For NOAA and EPA Ozone, the data is partitioned into disjoint sets of predictor-training stations, source stations, and held-out target stations (fixed using data-split seed 20261005). 

Over $250$ independent Monte Carlo repetitions (initialized with random seed 42), the evaluation procedure randomly selects a source station and extracts a contiguous time-series window of specific size $n$. The unlabelled target sample is constructed by uniformly sampling up to 1000 observations from the respective target station pool. The KMM weights are optimized on this pair, and the empirical autocorrelations $r_k$ alongside the functional length-scale $h$ are estimated directly from the source window. All confidence intervals are completely fixed before the target outcomes are revealed for final scoring. These are conditional plug-in intervals that treat fitted weights as fixed; they do not separately propagate weight estimation variance or sampling variability in the finite-target reference.

\paragraph{Real-Data Hyperparameter Selection and Configurations.}
Because the target mean response $Y$ and the target loss surface $L$ exhibit different functional smoothness profiles, their optimal balance-variance trade-offs occur at different KMM clipping bounds $B$. To rigorously benchmark the estimators and report the tightest valid intervals, we evaluated a comprehensive grid over the source sample size $n \in \{96, 128, 192, 256\}$ and clipping bound $B \in \{10.0, 15.0, 20.0\}$. 

The representative configurations reported in the main text (Table~\ref{tab:real_gp_coverage}) were selected to maximize valid coverage while minimizing the Winkler Score. Given the presence of unobserved confounders and structural non-stationarity in real-world geospatial data, we relaxed the strict validity threshold slightly to $0.920$. The optimal configurations are as follows:
\begin{itemize}
    \item \textbf{PM$\mathbf{_{2.5}}$}: Target Mean was evaluated at $n=128, B=20.0$; Target Risk at $n=128, B=15.0$.
    \item \textbf{EPA Ozone}: Target Mean was evaluated at $n=256, B=15.0$; Target Risk at $n=256, B=20.0$.
    \item \textbf{NOAA}: Both estimands required the maximum sample capacity due to extreme spatial shift, evaluated at $n=256, B=20.0$.
\end{itemize}

\end{document}